\documentclass[runningheads]{llncs}
\usepackage{graphicx}
\usepackage{microtype}
\usepackage{graphicx}
\usepackage{subfig}
\usepackage{caption}
\usepackage{amssymb,amsmath}
\usepackage{float}
\usepackage{bm}
\usepackage{booktabs}
\usepackage{hyperref}
\usepackage{algorithm}
\usepackage{algpseudocode}
\usepackage{algorithmicx}
\usepackage{xcolor}

\usepackage{xr}
\makeatletter
\newcommand*{\addFileDependency}[1]{% argument=file name and extension
  \typeout{(#1)}
  \@addtofilelist{#1}
  \IfFileExists{#1}{}{\typeout{No file #1.}}
}
\makeatother

\newcommand*{\myexternaldocument}[1]{%
    \externaldocument{#1}%
    \addFileDependency{#1.tex}%
    \addFileDependency{#1.aux}%
}
\myexternaldocument{supplementary}

\RequirePackage{fancyhdr}
\newcommand{\R}{\mathbb{R}}

\newcommand{\matr}[1]{\boldsymbol{#1}}
\newcommand\norm[1]{\left\lVert#1\right\rVert}

\renewcommand{\vec}[1]{\mathbf{#1}}
\newcommand{\mat}[1]{\mathbf{#1}}

\DeclareMathOperator{\argmin}{arg\, min}
\DeclareMathOperator{\sigmamin}{\sigma_{min}}
\def\R{\mathbb{R}}

\def\Prob{\mathbb{P}}

\def\e{\vec{e}}
\def\x{\vec{x}}
\def\b{\vec{b}}
\def\z{\vec{z}}

\def\a{\vec{a}}
\def\v{\vec{v}}
\def\n{\vec{n}}
\def\y{\vec{y}}
\def\matC{\matr{C}}
\def\matR{\matr{R}}
\def\matA{\matr{A}}
\def\matI{\matr{I}}
\def\matP{\matr{P}}
\def\matY{\matr{Y}}
\def\matZ{\matr{Z}}
\def\matW{\matr{W}}
\def\matG{\matr{G}}
\def\matX{\matr{X}}
\def\matM{\matr{M}}
\def\matB{\matr{B}}
\def\matU{\matr{U}}
\def\matQ{\matr{Q}}
\def\matV{\matr{V}}
\def\matS{\matr{S}}
\def\matE{\matr{E}}

\newtheorem{lem}{Lemma}

\newtheorem{thm}{Theorem}
\newtheorem{defn}{Definition}

\newtheorem{cor}{Corollary}

\begin{document}
\title{NoisyCUR: An algorithm for two-cost budgeted matrix completion}
\author{Dong Hu\inst{1} \and
Alex Gittens\inst{1} \and
Malik Magdon-Ismail\inst{1}}
% \author{Anonymous}
\authorrunning{D. Hu et al.}
\institute{Anonymous}
\institute{Rensselaer Polytechnic Institute, Troy, NY 12180, USA \\
\email{\{hud3,gittea\}@rpi.edu, magdon@cs.rpi.edu}}
\maketitle              % typeset the header of the contribution
\begin{abstract}
Matrix completion is a ubiquitous tool in machine learning and data analysis. Most work in this area has focused on the number of observations necessary to obtain an accurate low-rank approximation. In practice, however, the cost of observations is an important limiting factor, and experimentalists may have on hand multiple modes of observation with differing noise-vs-cost trade-offs. This paper considers matrix completion subject to such constraints: a budget is imposed and the experimentalist's goal is to allocate this budget between two sampling modalities in order to recover an accurate low-rank approximation. Specifically, we consider that it is possible to obtain low noise, high cost observations of individual entries or high noise, low cost observations of entire columns. We introduce a regression-based completion algorithm for this setting and experimentally verify the performance of our approach on both synthetic and real data sets. When the budget is low, our algorithm outperforms standard completion algorithms. When the budget is high, our algorithm has comparable error to standard nuclear norm completion algorithms and requires much less computational effort.

\end{abstract}

\section{Introduction}

Matrix completion (MC) is a powerful and widely used tool in machine learning, finding applications in information retrieval, collaborative filtering, recommendation systems, and computer vision. The goal is to recover a matrix $\matA \in \R^{m \times n}$ from only a few, potentially noisy, observations $\y \in \R^{d}$, where $d \ll mn$.

In general, the MC problem is ill-posed, as many matrices may give rise to the same set of observations. Typically the inversion problem is made feasible by assuming that the matrix from which the observations were generated is in fact low-rank, $\text{rank}(\matA) = r \ll \min\{m,n\}$. In this case, the number of degrees of freedom in the matrix is $(n+m)r$, so if the observations are sufficiently diverse, then the inversion process is well-posed.

In the majority of the MC literature, the mapping from the matrix to its observations, although random, is given to the user, and the aim is to design algorithms that minimize the sample complexity under these observation models. Some works have considered modifications of this paradigm, where the user designs the observation mapping themselves in order to minimize the number of measurements needed~\cite{YCHEN,Cur+,krishnamurthy2014power}.

\begin{figure}[!ht]%
    \centering
    \includegraphics[width=0.8\linewidth]{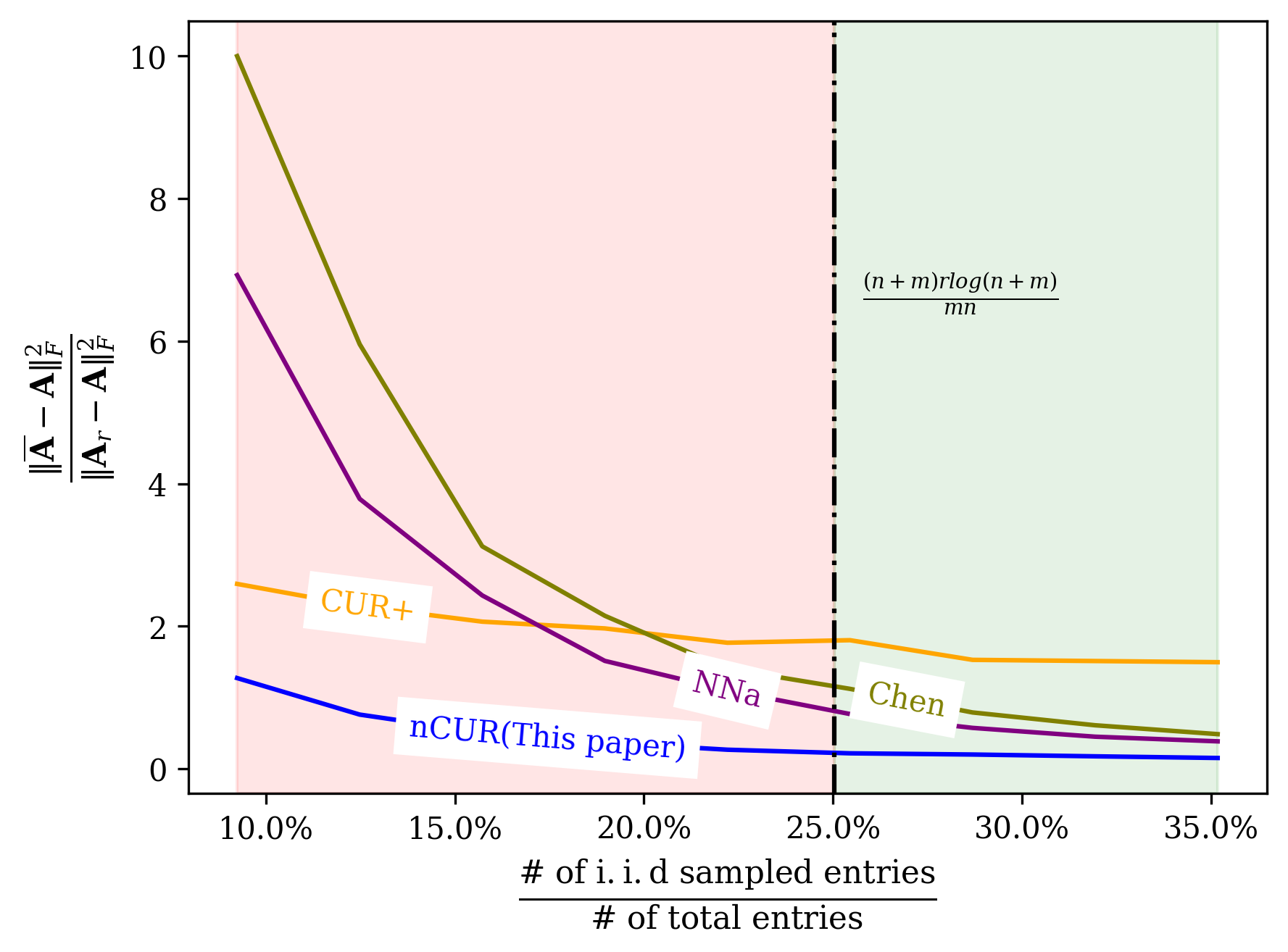}
    \caption{noisyCUR, a regression-based MC algorithm (CUR+), and two nuclear norm MC algorithms applied to the budgeted MC problem on a synthetic incoherent matrix $\matr A \in \R^{80\times 60}$. See Section~\ref{sec:empirical} for more details on the experimental setup. The performance of each method at each budget level is averaged over 10 runs, with hyper-parameters selected at each value of $d$ by cross-validation. Within the red region the budget is small enough that not enough entries can be sampled for nuclear norm methods to have theoretical support. In this regime, noisyCUR significantly outperforms the baseline algorithms. }%
    \label{fig:compare}%
    \end{figure}
    
This paper considers a budgeted variation of noisy matrix completion with competing observation models: the goal is to allocate a finite budget between the observation models to maximize the accuracy of the recovery. This setup is natural in the context of experimental design: where, given competing options, the natural goal of an experimenter is to spend their budget in a way that maximizes the accuracy of their imputations. Specifically, this work considers a two-cost model where either entire columns can be observed with high per-entry error but low amortized per-entry cost, or entries can be observed with low per-entry error but high per-entry cost. This is a natural model in, for instance, recommender systems applications, where one has a finite budget to allocate between incentivizing users to rate either entire categories of products or individual items. The former corresponds to inexpensive, high-noise measurements of entire columns and the latter corresponds to expensive, low-noise measurements of individual entries. 

The noisyCUR algorithm is introduced for this two-cost budgeted MC problem, and guarantees are given on its recovery accuracy. Figure~\ref{fig:compare} is illustrative of the usefulness of noisyCUR in this setting, compared to standard nuclear norm MC approaches and another regression-based MC algorithm. Even in low-budget settings where nuclear norm matrix completion is not theoretically justifiable, noisyCUR satisfies relative-error approximation guarantees. Empirical comparisons of the performance of noisyCUR to that of nuclear norm MC approaches on a synthetic dataset and on the Jester and MovieLens data sets confirm that noisyCUR can have significantly lower recovery error than standard nuclear norm approaches in the budgeted setting. Additionally, noisyCUR tolerates the presence of coherence in the row-space of the matrix, and is fast as its core computational primitive is ridge regression.

The rest of the paper is organized as follows. In Section~\ref{sec:problem} we introduce the two-cost model and discuss related works. Section~\ref{sec:algorithm} introduces the noisyCUR algorithm and provides performance guarantees; most proofs are deferred to the supplement. Section~\ref{sec:empirical} provides an empirical comparison of the performance of noisyCUR and baseline nuclear norm completion algorithms that demonstrates the superior performance of noisyCUR in the limited budget setting. Section~\ref{sec:conclusion} concludes the work.

\section{Problem Statement}
\label{sec:problem}
\subsection{Notation.} 
Throughout this paper, scalars are denoted by lowercase letters ($n$), vectors by bolded lowercase letters ($\x$), and matrices by bolded uppercase letters $\matA$. The spectral, Frobenius, and nuclear norms of $\matA \in \R^{m \times n}$ are written $\|\matA\|_2$, $\|\matA\|_F$, and $\|\matA\|_\star$ respectively, and its singular values are ordered in a decreasing fashion: $\sigma_1(\matA) \geq \cdots \geq \sigma_{\min\{m, n\}}(\matA)$. The smallest nonzero singular value of $\matA$ is denoted by $\sigma_{\text{min}}(\matA).$ The condition number of $\matA$ is taken to be the ratio of the largest singular value and the smallest \emph{nonzero} singular value, $\kappa_2(\matA)=\sigma_1(\matA)/\sigma_{\text{min}}(\matA).$ The orthogonal projection onto the column span of $\matA$ is denoted by $\matP_{\matA}$. Semi-definite inequalities between the positive-semidefinite matrices $\matA$ and $\matB$ are written, e.g., as $\matA \preceq \matB$. The standard Euclidean basis vectors are written as $\e_1$, $\e_2$, and so on.

\subsection{Problem formulation}

Given a limited budget $B$ with which we can pay to noisily observe a matrix $\matA \in \R^{m \times n}$, our goal is to construct a low-rank matrix $\overline{\matr A}$ that approximates $\matA$ well. There are two modes of observation: very noisy and cheap samples of entire columns of $\matA$, or much less noisy but expensive samples of individual entries of $\matA$.

The following parameters quantify this two-cost observation model:
\begin{itemize}
    \item Each low-noise sample of an individual entry costs $p_e$.
    \item Each high-noise sample of a column costs $p_c > 0$. Because columns are cheap to sample, $p_c < m p_e$.
    \item The low-noise samples are each observed with additive $\mathcal{N}(0, \sigma_e^2)$ noise.
    \item Each entry of the high-noise column samples is observed with additive $\mathcal{N}(0, \sigma_c^2)$ noise. Because sampling columns is noisier than sampling entries, $\sigma_c^2 > \sigma_e^2$.
\end{itemize}

\subsection{Related Work}

To the best of the authors' knowledge, there is no prior work on budgeted low-rank MC. Standard approaches to matrix completion, e.g.~\cite{recht2011simpler,keshavan2009matrix,hastie2015matrix} assume that the entries of the matrix are sampled uniformly at random, with or without noise. The most related works in the MC literature concern adaptive sampling models that attempt to identify more informative observations to reduce the sample complexity~\cite{YCHEN,krishnamurthy2014power,Cur+,krishnamurthy2013low,balcan2016noise}. One of these works is~\cite{YCHEN}, which estimates non-uniform sampling probabilities for the entries in $\matA$ and then samples according to these to reduce the sample complexity. The work~\cite{krishnamurthy2014power} similarly estimates the norms of the columns of $\matA$ and then samples entries according to these to reduce the sample complexity. The work~\cite{Cur+} proposes a regression-based algorithm for noiseless matrix completion that uses randomly sampled columns, rows, and entries of the matrix to form a low-rank approximation.

\section{The noisyCUR algorithm }
\label{sec:algorithm}

 \begin{algorithm}[t!]
 	\begin{algorithmic}[1]
 		\Require $d$, the number of column samples;
 		$s$, the number of row samples; $\sigma_e$, the noise level of the column samples; $\sigma_c$, the noise level of the row samples; and $\lambda$, the regularization parameter
 		\Ensure $\overline{\matA}$, approximation to $\matA$
		\State $\tilde{\matC} \gets \matC + \matE_c$, where the $d$ columns of $\matC$ are sampled uniformly at random with replacement from $\matA$, and the entries of $\matE_c$ are i.i.d. $\mathcal{N}(0, \sigma_c^2)$.
 		\State $\matU \gets $ an orthonormal basis for the column span of $\tilde{\matC}$
		\State $\bm{\ell}_i \gets \tfrac{1}{2} \|\e_i^T \matU\|_2^2/\|\matU\|_F^2 + \tfrac{1}{2m}$ for $i = 1, \ldots, m$
		\State $\matS \gets \text{SamplingMatrix}(\bm{\ell}, m, s)$, the sketching matrix\footnotemark~used to sample $s$ rows of $\matA$
		\State $\matY = \matS^T \matA + \matE_e$, where the entries of $\matE_e$ are i.i.d. $\mathcal{N}(0, \sigma_e^2)$.
		\State $\matX \gets \argmin_{\matZ} \|\matY - \matS^T \tilde{\matC} \matZ\|_F^2 + \lambda \|\matZ\|_F^2$
		\State $\overline{\matA} \gets \tilde{\matC}\matX$
		\State \Return $\overline{\matA}$
	\end{algorithmic}
	\caption{noisyCUR algorithm for completion of a low-rank matrix $\matA \in \R^{m \times n}$}
	 \label{alg:noisycur}
 \end{algorithm}
 \footnotetext{$\text{SamplingMatrix}(\bm{\ell}, m, s)$ returns a matrix $\matS \in \R^{m \times s}$ such that $\matS^T\matA$ samples and rescales, i.i.d.~with replacement, $s$ rows of $\matA$ with probability proportional to their shrinked leverage scores. See the supplement for details. } 
The noisyCUR (nCUR) algorithm, stated in Algorithm~\ref{alg:noisycur}, is a regression-based MC algorithm. The intuition behind noisyCUR is that, if the columns of $\matA$ were in general position and $\matA$ were exactly rank-$r$, then one could recover $\matA$ by sampling exactly $r$ columns of $A$ and sampling $r$ entries from each of the remaining columns, then using regression to infer the unobserved entries of the partially observed columns. 

noisyCUR accomodates the presence of noise in the column and entry observations, and the fact that $\matA$ is not low-rank. It samples $d$ noisy columns from the low-rank matrix $\matA \in \R^{m \times n}$, with entry-wise error that has variance $\sigma_c^2$ and collects these columns in the matrix $\tilde{\matC} \in \R^{m \times d}$. It then uses $\tilde{\matC}$ to approximate the span of $\matA$ and returns an approximation of the form $\overline{\matA} = \tilde{\matC} \matX$.

The optimal coefficient matrix $\matX$ would require $\matA$ to be fully observed. The sample complexity is reduced by instead sampling $s$ rows noisily from $\matA$, with entry-wise error that has variance $\sigma_e^2$, to form the matrix $\matY$. These rows are then used to estimate the coefficient matrix $\matX$ by solving a ridge-regression problem; ridge regression is used instead of least-squares regression because in practice it better mitigates the noisiness of the observations. The rows are sampled according to the shrinked leverage scores \cite{ma2015statistical,wang2018sketched} of the rows of $\tilde{\matC}$.

The cost of observing the entries of $\matA$ needed to form the approximation $\overline{\matA}$ is $d p_c + snp_e$, so the number of
column observations $d$ and the number of row observations $s$ are selected to use the entire budget, $B = d p_c + s n p_e$.

Our main theoretical result is that when $\matA$ is column-incoherent, has low condition number and is sufficiently dense, noisyCUR can exploit the two sampling modalities to achieve additive error guarantees in budget regimes where the relative error guarantees of nuclear norm completion do not apply.

First we recall the definition of column-incoherence.

\begin{defn} If $\matA \in \R^{m \times n}$ and $\matV \in \R^{m \times r}$ is an orthonormal basis for the row space of $\matA$, then the column leverage scores of $\matA$ are given by 
	\[
		\ell_i = \|\e_i^T \matV\|_2^2
		\quad \text{ for } i=1,\ldots,n.
	\]
The column coherence of $\matA$ is the maximum of the column leverage scores of $\matA$, and $\matA$ is said to have a $\beta$-incoherent column space if its column coherence is smaller than $\beta \frac{r}{n}$.
\end{defn}

\begin{thm}
\label{thm:noisycur_guarantee}
Let $\matA \in \R^{m \times n}$ be a rank-$r$ matrix with $\beta$-incoherent column space. Assume that $\matA$ is dense: there is a $c > 0$ such that at least half \footnote{This fraction can be changed, with corresponding modification to the sample complexities $d$ and $s$.} of the entries of $\matA$ satisfy $|a_{ij}| \geq c$.

Fix a precision parameter $\varepsilon \in (0, 1)$ and invoke Algorithm~\ref{alg:noisycur} with 
\[
d \geq \max\left\{ \frac{6 + 2 \varepsilon}{3\varepsilon^2} \beta r \log \frac{r}{\delta},  \frac{8(1 + \delta)^2}{c^2(1 - \varepsilon)\varepsilon} r \kappa_2(\matA)^2 \sigma_c^2\right\} 
\text{ and }
s \geq \frac{6 + 2 \varepsilon}{3\varepsilon^2} 2d \log \frac{d}{\delta}.
\]
The returned approximation satisfies
\[
\|\matA - \overline{\matA}\|_F^2 \leq \left(\gamma + \varepsilon + 40 \frac{\varepsilon}{1 - \varepsilon}\right) \|\matA\|_F^2 + 12 \varepsilon \left(\frac{\sigma_e^2}{\sigma_c^2}\right) d \sigma_r^2(\matA)
\]
with probability at least $0.9 - 2\delta - 2\exp(\tfrac{-(m-r)\delta^2}{2}) - \exp(\tfrac{-sn}{32})$. Here,
\[
\gamma \leq 2 \left(\frac{1 + \varepsilon}{1 - \varepsilon}\right) \left[ \frac{\lambda}{(1 + \varepsilon) \left( \frac{1}{2}\sqrt{m-r} - \sqrt{d} \right)^2\sigma_c^2 + \lambda}\right]^2.
\]
\end{thm}
The proof of this result is deferred.

\paragraph{\textbf{Comparison of nuclear norm and noisyCUR approximation guarantees.} }Theorem~\ref{thm:noisycur_guarantee} implies that, if $\matA$ has low condition number, is dense and column-incoherent, and the regularization parameter $\lambda$ is selected to be $o((\sqrt{m-r}-\sqrt{d})^2 \sigma_c)$, then a mixed relative-additive bound of the form
\begin{align}
\label{eqn:ourguarantee}
\|\matA - \overline{\matA}\|_F^2 & \leq \varepsilon^\prime \|\matA\|_F^2 + \varepsilon^{\prime\prime} \tilde{O}(r) \sigma_r^2(\matA) \notag \\
 & = \varepsilon^\prime \|\matA\|_F^2 + \varepsilon^{\prime\prime} \tilde{O}(\|\matA\|_F^2) 
\end{align}
holds with high probability for the approximation returned by Algorithm~\ref{alg:noisycur}, where $\varepsilon^\prime$ and $\varepsilon^{\prime\prime}$ are $o(1)$, when $d$ and $s$ are $\tilde{\Omega}(r)$.

By way of comparison, the current best guarantees for noisy matrix completion using nuclear norm formulations state that if $d = \Omega((n + m)r \log (n+m))$ entries are noisily observed with noise level $\sigma_e^2$, then a nuclear norm formulation of the matrix completion problem yields an approximation $\overline{\matA}$ that satisfies
\begin{equation*}
\|\matA - \overline{\matA}\|_F^2 = O\left( \frac{\sigma_e^2}{\sigma_r^2(\matA) }  \frac{nm}{r} \right) \|\matA\|_F^2
\end{equation*}
with high probability~\cite{chen2019noisy,balcan2019nonconvex}. The conditions imposed to obtain this guarantee~\cite{chen2019noisy} are that $\matA$ has low condition number and that both its row and column spaces are incoherent. If we additionally require that $\matA$ is dense, so that the assumptions applied to both algorithms are comparable, then $\|\matA\|_F^2 = \Omega(mn)$ and the guarantee for nuclear norm completion becomes
\begin{equation}
\label{eqn:bestnnguarantee}
\|\matA - \overline{\matA}\|_F^2 = O(\sigma_e^2 \kappa_2^2(\matA))\|\matA\|_F^2.
\end{equation}

%In practice, when the budget allows enough high precision noisy samples to be drawn, the nuclear norm approach gives lower error approximations. 

\paragraph{\bfseries Comparison of budget requirements for nuclear norm completion and noisy CUR.}
For nuclear norm completion approaches to assure guarantees of the form \eqref{eqn:bestnnguarantee} it is necessary to obtain $\Omega((n+m)r\log(n+m))$ high precision noisy samples~\cite{candes2010power}, so the budget required is
\[
B_{NN} = \Omega((n+m)r \log(n+m)p_e).
\]
When $B_{NN}$ exceeds the budget $B$, there is no theory supporting the use of a mix of more expensive high and cheaper low precision noisy measurements. 
%Even empirically (see Section~\ref{sec:experiment}), straightforward attempts to address the use of two sampling modalities in the nuclear norm framework are unsuccessful. 

The noisyCUR algorithm allows exactly such a mix: the cost of
obtaining the necessary samples is
\[
B_{nCUR} = dp_c + snp_e = \tilde{\Omega}(r)p_c + \tilde{\Omega}_r(nr)p_e,
\]
where the notation $\tilde{\Omega}_r(\cdot)$ is used to indicate that the omitted logarithmic factors depend only on $r$. It is evident that
\[
B_{nCUR} < B_{NN},
\]
so the noisy CUR algorithm is applicable in budget regimes where nuclear norm completion is not.

\subsection{Proof of Theorem~\ref{thm:noisycur_guarantee}}

Theorem~\ref{thm:noisycur_guarantee} is a consequence of two structural results that are established in the supplement. 

The first result states that if $\text{rank}(\matC) = \text{rank}(\matA)$ and the bottom singular value of $\matC$ is large compared to $\sigma_c$, then the span of $\tilde{\matC}$ will contain a good approximation to $\matA$. 

\begin{lem}
	\label{lem:structural_noisy_approx}
	Fix an orthonormal basis $\matU \in \R^{m \times r}$ and consider $\matA \in \R^{m \times n}$ and $\matC \in \R^{m \times d}$ with factorizations $\matA = \matU \matM$ and $\matC = \matU \matW$, where both $\matM$ and $\matW$ have full row rank. Further, let $\tilde{\matC}$ be a noisy observation of $\matC$, that is, let $\tilde{\matC} = \matC + \matG$ where the entries of $\matG$ are i.i.d. $\mathcal{N}(0, \sigma_{c}^2)$. If $\sigmamin(\matC) \geq 2(1 + \delta) \sigma_c \sqrt{m/\varepsilon}$, then 
	\[
		\|(\matI - \matP_{\tilde{\matC}})\matA\|_F^2 \leq \varepsilon \|\matA\|_F^2
		\]
	with probability at least $1 - \exp\left(\frac{-m\delta^2}{2}\right)$.
\end{lem}

Recall the definition of a $(1\pm \varepsilon)$-subspace embedding.

\begin{defn}[Subspace embedding~\cite{woodruff2014sketching}]
	Let $\matA \in \R^{m \times n}$ and fix $\varepsilon \in (0,1)$. A matrix $\matS \in \R^{m \times s}$ is a $(1 \pm \varepsilon)$-subspace embedding for $\matA$ if
	\[
		(1 - \varepsilon) \|\x\|_2^2 \leq \|\matS^T\x\|_2^2 \leq (1 + \varepsilon) \|\x\|_2^2
	\]
	for all vectors $\x$ in the span of $\matA,$ or equivalently, if
	\[
	(1 - \varepsilon) \matA^T \matA \preceq \matA^T \matS \matS^T \matA \preceq (1 + \varepsilon) \matA^T \matA.
	\]
	Often we will use the shorthand ``subspace embedding'' for $(1 \pm \varepsilon)$-subspace embedding.
\end{defn}

The second structural result is a novel bound on the error of sketching using a subspace embedding to reduce the cost of ridge regression, when the target is noisy.

\begin{cor}
  \label{cor:structural_sketched_proximal_regularized_LS}
  Let $\tilde{\matC} \in \R^{m \times d}$, where $d \leq m$, and $\tilde{\matA} = \matA + \matE$ be matrices, and let $\matS$ be an $(1 \pm \varepsilon)$-subspace embedding for $\tilde{\matC}$. If
	\[
    \matX = \argmin_{\matZ} \|\matS^T(\tilde{\matA} - \tilde{\matC} \matZ)\|_F^2 + \lambda \|\matZ\|_F^2,
	\]
	then 
	\begin{multline*}
    \|\matA - \tilde{\matC} \matX\|_F^2 \leq \|(\matI - \matP_{\tilde{\matC}})\matA\|_F^2 + \gamma \|\matP_{\tilde{\matC}} \matA\|_F^2
       + \frac{4}{1 - \varepsilon}\|\matS^T \matE\|_F^2 + \frac{4}{1 - \varepsilon}\|\matS^T (\matI - \matP_{\tilde{\matC}})\matA\|_F^2,
	\end{multline*}
  where $\gamma =  2 \left(\frac{1 + \varepsilon}{1 - \varepsilon} \right) \left( \frac{\lambda}{(1 + \varepsilon) \sigma_d(\tilde{\matC})^2 + \lambda} \right)^{2}.$
\end{cor}

Corollary~\ref{cor:structural_sketched_proximal_regularized_LS} differs significantly from prior results on the error in sketched ridge regression, e.g.~\cite{avron2017sharper,wang2018sketched}, in that: (1) it bounds the \emph{reconstruction error} rather than the \emph{ridge regression objective}, and (2) it considers the impact of noise in the target. This result follows from a more general result on sketched noisy proximally regularized least squares problems, stated as Theorem~\ref{supp-thm:sketched_regularized_LS} in the supplement.

Together with standard properties of Gaussian noise and subspace embeddings, these two results deliver Theorem~\ref{thm:noisycur_guarantee}.

\begin{proof}[Proof of Theorem~\ref{thm:noisycur_guarantee}]
The noisyCUR algorithm first forms the noisy column samples $\tilde{\matC} = \matC + \matE_c$, where $\matC = \matA \matM$.  The random matrix $\matM \in \R^{n \times d}$ selects $d$ columns uniformly at random with replacement from the columns of $\matA$, and the entries of $\matE_c \in \R^{m \times d}$ are i.i.d.\ $\mathcal{N}(0, \sigma_c^2)$. It then solves the sketched regression problem
\[
\matX = \argmin_{\matZ} \|\matS^T(\tilde{\matA} - \tilde{\matC} \matZ)\|_F^2 + \lambda \|\matZ\|_F^2,
\]
and returns the approximation $\overline{\matA} = \tilde{\matC} \matX$. Here $\tilde{\matA} = \matA + \matE_e$, where $\matE_e \in \R^{m \times n}$ comprises i.i.d $\mathcal{N}(0, \sigma_e^2)$ entries, and the sketching matrix $\matS \in \R^{m \times s}$ samples $s$ rows using the shrinked leverage scores of $\tilde{\matC}$.

By~\cite[Appendix A.1.1]{wang2018sketched}, $\matS$ is a subspace embedding for $\tilde{\matC}$ with failure probability at most $\delta$ when $s$ is as specified. Thus Corollary~\ref{cor:structural_sketched_proximal_regularized_LS} applies and gives that
\begin{align*}
    \|\matA - \tilde{\matC} \matX\|_F^2 & \leq \|(\matI - \matP_{\tilde{\matC}})\matA\|_F^2 + \gamma^\prime \|\matP_{\tilde{\matC}} \matA\|_F^2 \\
      & \quad \quad + \frac{4}{1 - \varepsilon}\|\matS^T \matE\|_F^2 + \frac{4}{1 - \varepsilon}\|\matS^T (\matI - \matP_{\tilde{\matC}})\matA\|_F^2 \\
       & = T_1 + T_2 + T_3 + T_4,
\end{align*}
where $\gamma^\prime =  2 \left(\frac{1 + \varepsilon}{1 - \varepsilon} \right) \left( \frac{\lambda}{(1 + \varepsilon) \sigma_d(\tilde{\matC})^2 + \lambda} \right)^{2}.$ We now bound the four terms $T_1$, $T_2$, $T_3$, and $T_4$. 

To bound $T_1$, note that by~\cite[Lemma 13]{wang2016towards}, the matrix $\sqrt{\tfrac{n}{d}} \matM$ is a subspace embedding for $\matA^T$ with failure probability at most $\delta$ when $d$ is as specified. This gives the semidefinite inequality $\frac{n}{d} \matC \matC^T = \frac{n}{d} \matA \matM \matM^T \matA^T \succeq (1 - \varepsilon) \matA\matA^T$, which in turn gives that
\begin{align*}
\sigma_{r}^2(\matC) & \geq (1 - \varepsilon) \frac{d}{n} \sigma_r^2(\matA) 
\geq \frac{8(1+\delta)^2}{c^2 \varepsilon} \frac{r}{n} \|\matA\|_2^2 \sigma_c^2 \\
& \geq \frac{8(1+\delta)^2}{c^2 \varepsilon n} \|\matA\|_F^2 \sigma_c^2 \geq 4(1+ \delta)^2\frac{m}{\varepsilon} \sigma_c^2. 
\end{align*}
The second inequality holds because 
\begin{equation}
\label{eqn:lowerbound_bottom_singval}
d \geq \frac{8(1 + \delta)^2}{c^2(1 - \varepsilon)\varepsilon} r \kappa_2(\matA)^2 \sigma_c^2 \quad\text{implies}\quad 
\sigma_r^2(\matA) \geq \frac{8(1 + \delta)^2}{c^2(1 - \varepsilon)\varepsilon} \frac{r}{d} \|\matA_2\|^2 \sigma_c^2
\end{equation}
The third inequality holds because $r \|\matA\|_2^2$ is an overestimate of $\|\matA_F\|_2^2$. The final inequality holds because the denseness of $\matA$ implies that $\|\matA\|_F^2 \geq \tfrac{1}{2} c^2 mn$.

Note also that the span of $\matC = \matA \matM$ is contained in that of $\matA$, and since $\frac{n}{d} \matC \matC^T \succeq (1 - \varepsilon) \matA \matA^T$, in fact $\matC$ and $\matA$ have the same rank and therefore span the same space. Thus the necessary conditions to apply Lemma~\ref{lem:structural_noisy_approx} are satisfied, and as a result, we find that
\[
T_1 \leq \varepsilon \|\matA\|_F^2
\]
with failure probability at most $\exp(-\tfrac{m\delta^2}{2}).$

Next we bound $T_2$. Observe that $\|\matP_{\tilde{\matC}} \matA\|_F^2 \leq \|\matA\|_F^2$. Further, by Lemma~\ref{supp-lem:perturbed_bottom_sing_val} in the supplement,
\[
\sigma_d(\tilde{\matC}) \geq \left(\frac{1}{2}\sqrt{m-r} - \sqrt{d}\right)^2 \sigma_c^2
\]
with failure probability at most $\exp(\tfrac{-(m-r)\delta^2}{2})$. This allows us to conclude that
\[
T_2 \leq \gamma \|\matA\|_F^2,
\]
where $\gamma$ is as specified in the statement of this theorem.

To bound $T_3$, we write
\begin{align*}
T_3 & = \frac{4}{1-\varepsilon}\|\matS^T \matP_{\matS} \matE\|_F^2 \leq \frac{4}{1-\varepsilon}\|\matS\|_2^2 \|\matP_{\matS} \matE\|_F^2 \\
& \leq \frac{8}{1-\varepsilon} \frac{m}{s} \|\matQ^T \matE\|_F^2,
\end{align*}
where $\matQ$ is an orthonormal basis for the span of $\matS$.
The last inequality holds because~\cite[Appendix A.1.2]{wang2018sketched} shows that $\|\matS\|_2^2 \leq 2\tfrac{m}{s}$ always. Finally, note that $\matQ$ has at most $s$ columns, so in the worst case $\matQ^T\matE$ comprises $sn$ i.i.d.\ $\mathcal{N}(0, \sigma_e^2)$ entries. A standard concentration bound for $\chi^2$ random variables with $sn$ degrees of freedom~\cite[Example 2.11]{wainwright2019high} guarantees that
\[
\|\matQ^T \matE\|_F^2 \leq \frac{3}{2} sn \sigma_e^2 
\]
with failure probability at most $\exp(\tfrac{-sn}{32})$. We conclude that, with the same failure probability,
\[
T_3 \leq \frac{12}{1- \varepsilon} m n \sigma_e^2. 
\]
Now recall \eqref{eqn:lowerbound_bottom_singval}, which implies that
\begin{align*}
\varepsilon(1-\varepsilon) d\sigma_r^2(\matA) 
& \geq \frac{8(1 + \delta)^2}{c^2} r \|\matA_2\|^2 \sigma_c^2 
\geq \frac{8(1 + \delta)^2}{c^2} \|\matA\|_F^2 \sigma_c^2 \\
& \geq 4(1 + \delta)^2 m n \sigma_c^2 \geq m n \sigma_c^2.
\end{align*}
It follows from the last two displays that
\[
T_3 \leq 12 \varepsilon \left(\frac{\sigma_e^2}{\sigma_c^2} \right) d \sigma_r^2(\matA).
\]

The bound for $T_4$ is an application of Markov's inequality. In particular, it is readily verifiable that $\mathbb{E}[\matS\matS^T] = \matI$, which implies that
\[
\mathbb{E}T_4 = \frac{4}{1- \varepsilon}\|(\matI - \matP_{\tilde{\matC}})\matA \|_F^2 = \frac{4}{1- \varepsilon}T_1 \leq \frac{4\varepsilon}{1- \varepsilon}\|\matA\|_F^2.
\]
The final inequality comes from the bound $T_1 \leq \varepsilon \|\matA\|_F^2$ that was shown earlier. Thus, by Markov's inequality, 
\[
T_4 \leq \frac{40\varepsilon}{1- \varepsilon}\|\matA\|_F^2
\]
with failure probability at most $0.1$.

Collating the bounds for $T_1$ through $T_4$ and their corresponding failure probabilities gives the claimed result.
\end{proof}

\section{Empirical Evaluation}
\label{sec:empirical}
In this section we investigate the performance of the noisyCUR method on a small-scale synthetic data set and on the Jester and MovieLens data sets. We compare with the performance of three nuclear norm-based algorithms in a low and a high-budget regime.

\subsection{Experimental setup} 
Four parameters are manipulated to control the experiment setup: 
\begin{enumerate}
	\item The budget, taken to be of the size $B = c_0mrp_e$ for some constant positive integer $c_0$. This choice ensures that the $O((n+m)r \log(n+m))$ high precision samples needed for nuclear norm completion methods cannot be obtained.  %The reason for this configuration is giving this budget we cannot sample $O(mr\text{ log}^2m)$ entry samples which is proved to be sufficient to recover the matrix with reasonable error bound using NNM-based methods like \cite{CANDES}(E. J. Cand$\grave{e}$s, 2009) and \cite{YCHEN}(Y. Chen et al, 2014). 
	\item The ratio of the cost of sampling a column to that of individually sampling each entry in that column, $\alpha = \frac{p_c}{m p_e}$. For all three experiments, we set $\alpha = 0.2$.	
	\item The entry sampling noise level $\sigma_e^2$.
	\item The column sampling noise level $\sigma_c^2$. 
\end{enumerate}
	Based on the signal-to-noise ratio between the matrix and the noise level of the noisiest observation model, $\sigma_c^2$, we classify an experiment as being high noise or low noise. The entry-wise signal-to-noise ratio is given by 
	\[
	SNR = \frac{\|\matA\|_F^2}{mn\sigma_c^2}.
	\]
	High SNR experiments are said to be low noise, while those with low SNR are said to be high noise. 
	%When $SNR = O(1/\alpha)$ the experiment is said to have high noise, and when $SNR = \omega(1/\alpha)$ the experiment is said to have low noise.
%			\item S/N ratio dB S/N ratio:1 Picture quality\\
%				60 dB 1,000 Excellent, no noise apparent\\
%				50 dB 316 Good, small noise, qual good.\\
%				40 dB 100 Reasonable, fine grain/snow.\\
%				30 dB 32 Poor pic. great deal of noise.\\
%				20 dB 10 Unusable picture.
	%The rationale for this nomenclature is that  is the cheaper you can sample a column the higher variance of the noise should be for the entries in that sampled column.

 \subsection{Methodology: noisyCUR and the baselines}

We compare to three nuclear norm-based MC algorithms, as nuclear norm-based approaches are the most widely used and theoretically investigated algorithms for low-rank MC. We additionally compare to the CUR+ algorithm of~\cite{Cur+} as it is, similarly to noisyCUR, a regression-based MC algorithm. 

To explain the baselines, we introduce some notation. Given a set of indices $\Omega$, the operator $\mathcal{P}_{\Omega} : \R^{m \times n} \rightarrow \R^{m \times n}$ returns a matrix whose values are the same as those of the input matrix on the indices in $\Omega$, and zero on any indices not in $\Omega$. The set $\Omega_s$ below comprises the indices of entries of $\matA$ sampled with high accuracy, while $\Omega_c$ comprises the indices of entries of $\matA$ sampled using the low accuracy column observation model. 

\paragraph{\bfseries (nCUR)} Given the settings of the two-cost model, the noisyCUR algorithm is employed by selecting a value for $d$, the number of noisy column samples; the remaining budget then determines $s$, the number of rows that are sampled with high precision. Cross-validation is used to select the regularization parameter $\lambda$.

\paragraph{\bfseries (CUR+)} The CUR+ algorithm is designed for noiseless matrix completion~\cite{Cur+}; it is adapted in a straightforward manner to our setting. Now $d$ is the number of noisy row and column samples, and $d/2$ columns and $d/2$ rows are sampled uniformly with replacement from $\matA$ and noisily observed to form column and row matrices $\matC$ and $\matR$. The remaining budget is used to sample entries to form $\Omega_e$ and ${\matr A}_{\text{obs}}$, the partially observed matrix which contains the observed noisy entry samples and is zero elsewhere. The CUR+ algorithm then returns the low-rank approximation $\overline{\matA} = \matC \matU \matR$, where $\matU$ is obtained by solving
\[
\matU = \argmin \| \mathcal{P}_{\Omega_e}(\matA_{\text{obs}} - \matC \matU \matR)\|_F^2.
\]

\paragraph{\bfseries (NNa)} The first of the nuclear norm baselines is the formulation introduced in~\cite{CANDESNOISE}, which forms the approximation
\begin{equation} \label{e2.1}
		\begin{aligned}
		\overline{\matA} = & \argmin_{\matZ} \norm{\matZ}_\star \\
			              & \mkern15mu \textrm{s.t.} \norm{\mathcal{P}_{\Omega_e}(\matZ-\matA_{\text{obs}})}_F\leq \delta, \quad (i,j)\in \Omega_e.
		\end{aligned}	
\end{equation} 
All of the budget is spend on sampling entries to form $\Omega_e$ and ${\matr A}_{\text{obs}}$, the partially observed matrix which contains the observed noisy entry samples and is zero elsewhere. Thus the performance of this model is a constant independent of $d$ in the figures. The hyperparameter $\delta$ is selected through cross-validation. This baseline is referred to as NNa (nuclear norm minimization for all entries) in the figures.

\paragraph{\bfseries (NNs)} The second nuclear norm baseline is a nuclear norm formulation that penalizes the two forms of observations separately, forming the approximation
	\begin{equation} \label{e2.2}
		\begin{aligned}
		\overline{\matA} = & \argmin_{\matZ} \norm{{ \matZ}}_\star \\
			& \textrm{s.t.} 
 			\;  \norm{\mathcal P_{\Omega_c}(\matZ-{\matr A}_{\text{obs}})}_F^2\leq C_1d m\sigma_c^2\\
 			&\; \ \quad \norm{\mathcal P_{\Omega_e}(\matZ-{\matr A}_{\text{obs}})}_F^2 \leq C_2f\sigma_e^2
		\end{aligned}
		\end{equation}
where $C_1$ and $C_2$ are parameters, and again ${\matr A}_{\text{obs}}$ is the partially observed matrix which contains the observed noisy column samples and entry samples and is zero elsewhere. As with the noisyCUR method, given a value of $d$, the remaining budget is spent on sampling $s$ rows with high precision. The hyperparameters $C_1$ and $C_2$ are selected through cross-validation. This baseline is referred to as NNs (nuclear norm split minimization) in the figures.

\paragraph{\bfseries (Chen)} The final nuclear norm baseline is an adaptation of the two-phase sampling method of~\cite{YCHEN}. This method spends a portion of the budget to uniformly at random sample entries to estimate leverage scores, then uses the rest of the budget to sample entries according to the leverage score and reconstructs from the union of these samples using the same optimization as NNa. The performance is therefore independent of $d$. This baseline is referred to as Chen in the figures.

The details of cross-validation of the parameters for the nuclear norm methods are omitted because there are many relevant hyperparameters: in addition to the  constraint parameters in the optimization formulations, there are important hyperparameters associated with the ADMM solvers used (e.g., the Lagrangian penalty parameters).

\subsection{Synthetic Dataset}
    Figure~\ref{fig:synthetic} compares the performance of the baseline methods and noisyCUR on an incoherent synthetic data set $\matA \in \R^{80 \times 60}$ generated by sampling a matrix with i.i.d.~$\mathcal{N}(5, 1)$ entries and taking its best rank four approximation. For each value of $d$, the regularization parameter $\lambda$ of noisyCUR is selected via cross-validation from 500 logarithmically spaced points in the interval $(10^{-4}, 10)$.

	\begin{figure}[h!]
    \centering
    \subfloat{\includegraphics[width=0.45\linewidth]{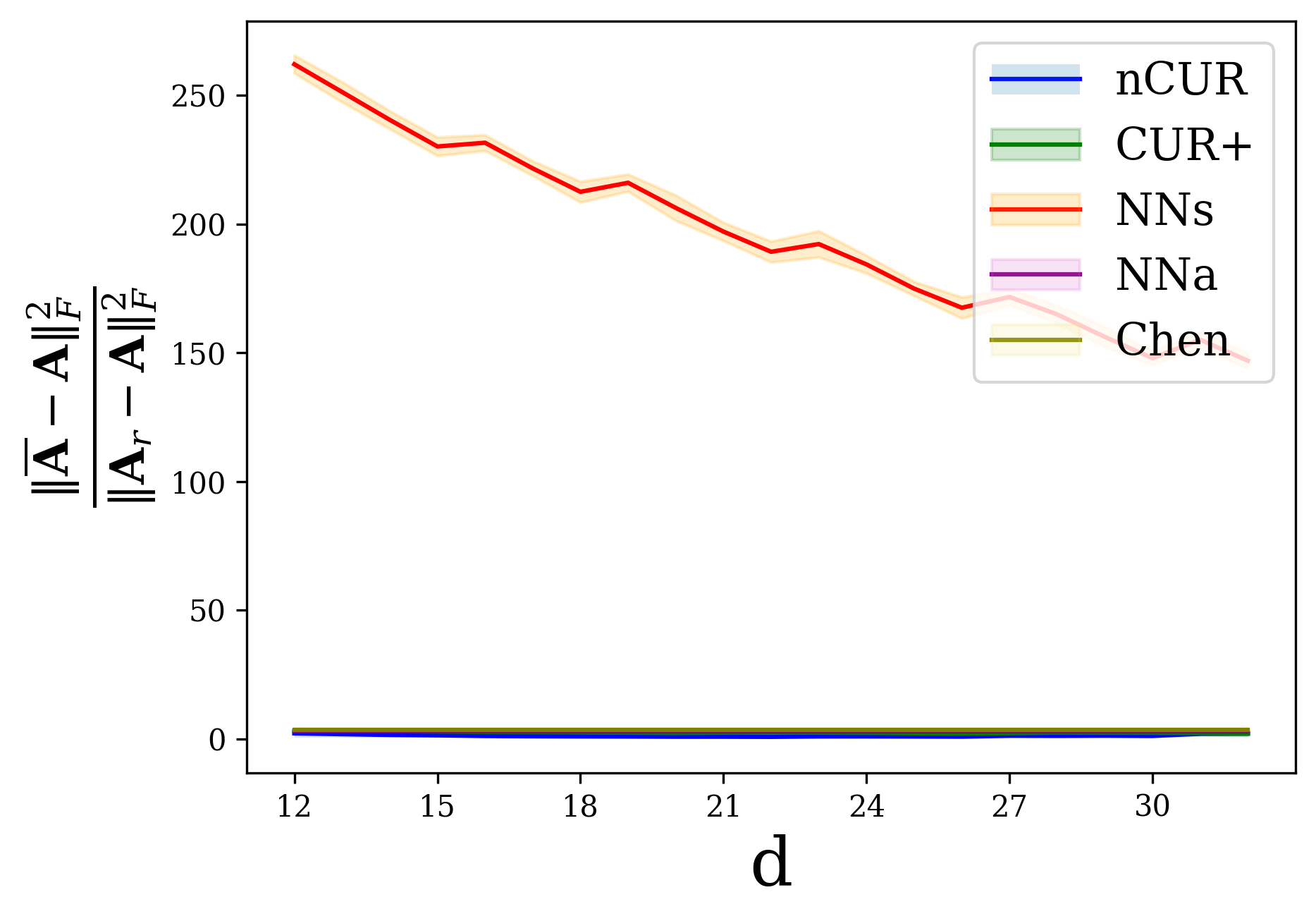} }%
    \subfloat{\includegraphics[width=0.45\linewidth]{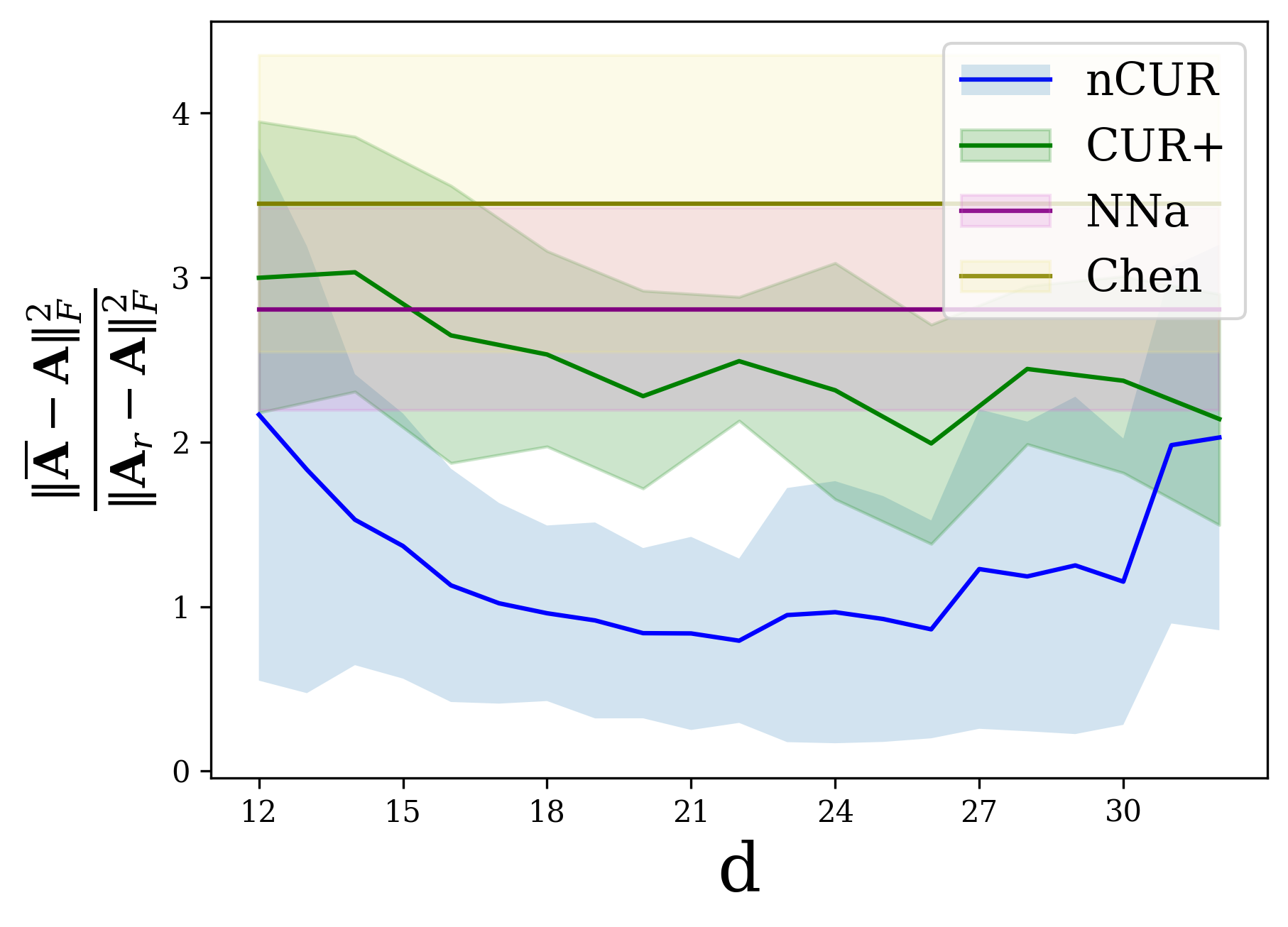} }
    
    \subfloat{\includegraphics[width=0.45\linewidth]{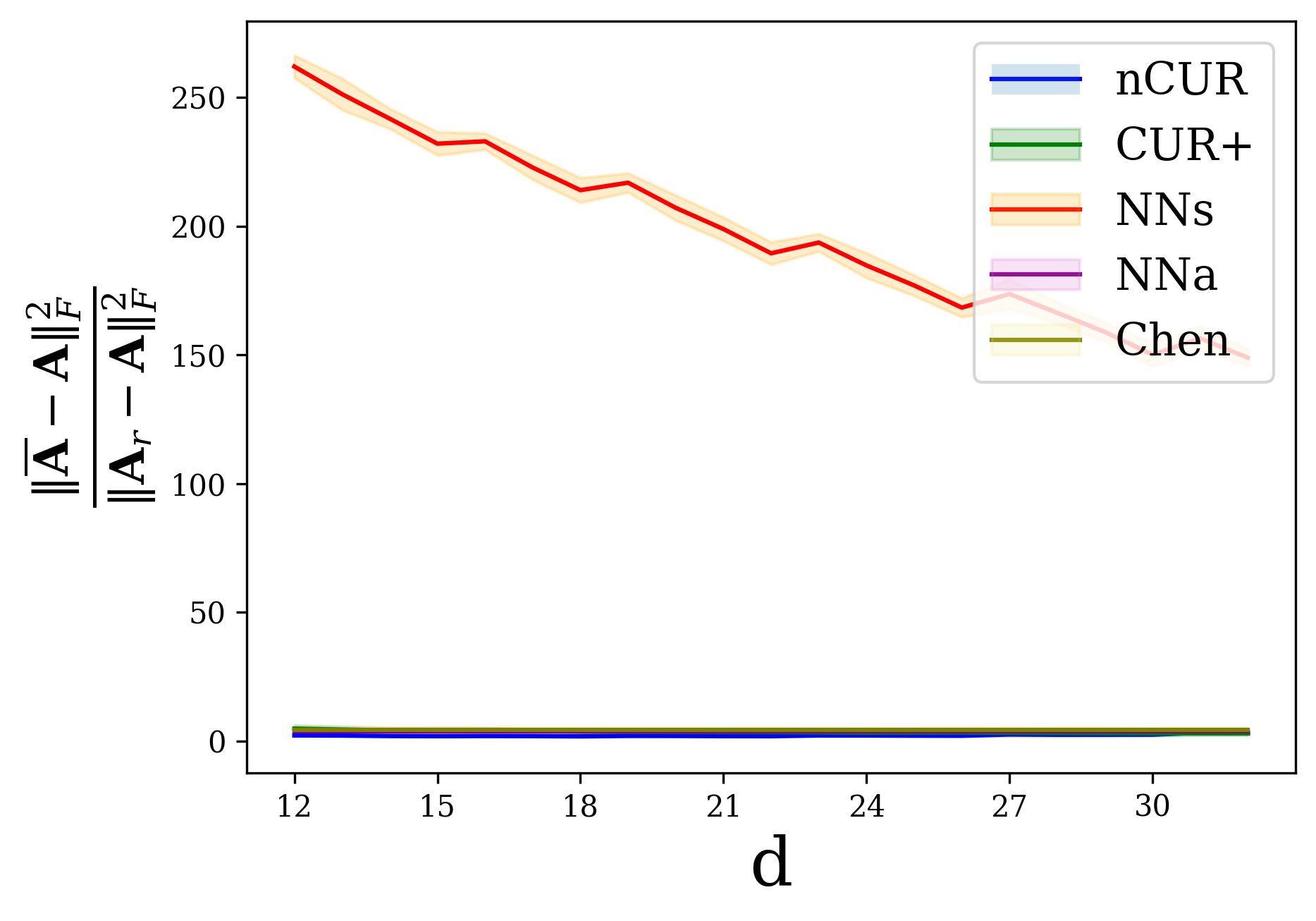} }%
    \subfloat{\includegraphics[width=0.45\linewidth]{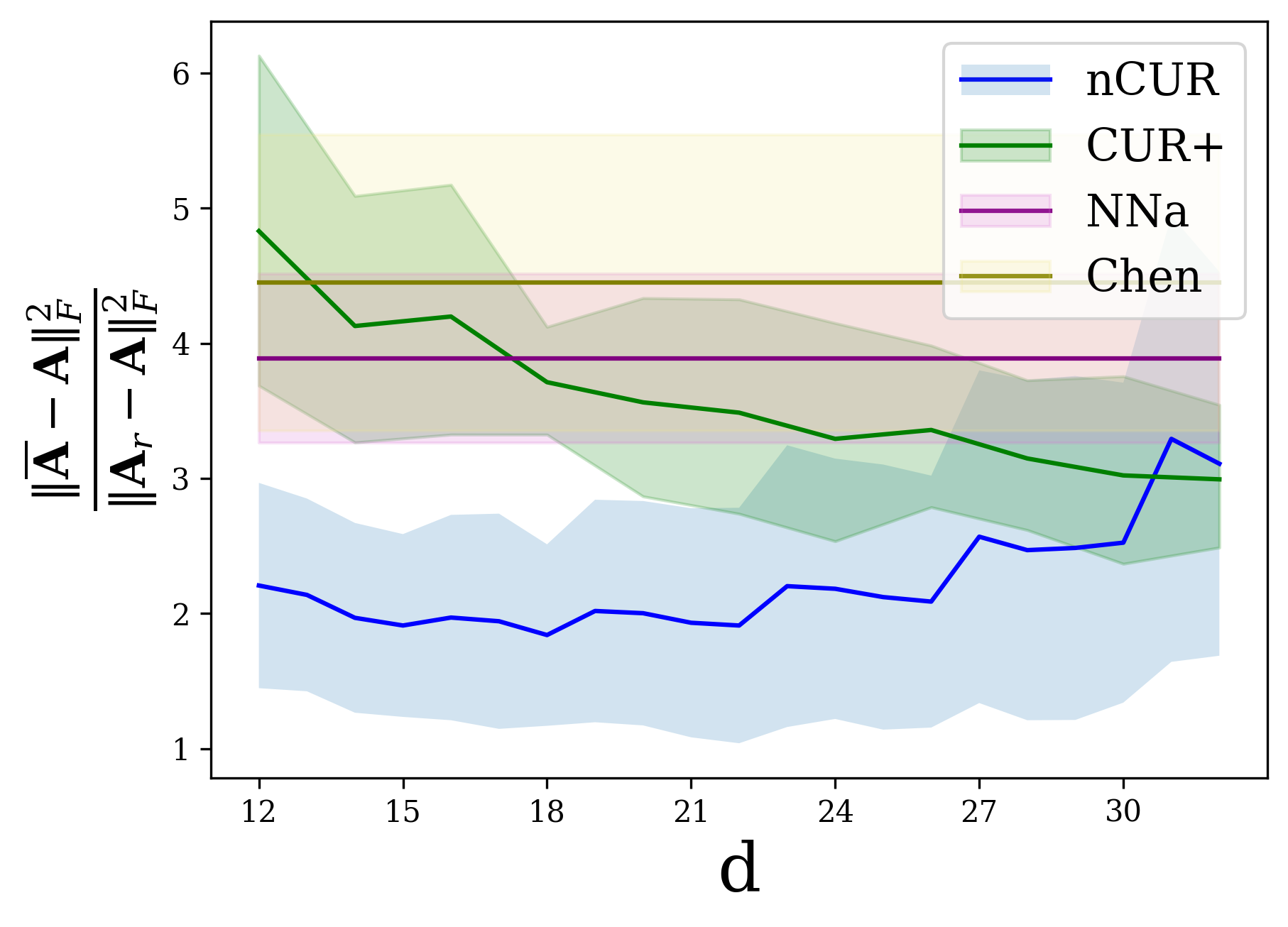} }%
    
    \caption{Performance on the synthetic dataset. $\beta=15\%$ for all plots. The noise level is low in the top plots: $\sigma_e^2=0.01,\sigma_c^2=0.05$. The upper left plot shows all methods, while the upper right plot removes the NNs method to facilitate comparison of the better performing methods. 
    In the bottom two plots, the noise level is higher: $\sigma_e^2=0.04,\sigma_c^2=0.2$. Similarly, the bottom left plot shows all methods, while the bottom right removes the NNs method. Each point in the plots is the average of 100 runs.}%
    \label{fig:synthetic}%
    \end{figure}	
    
\subsection{Jester}

\begin{figure}[htb!]%
    \centering
    \subfloat{\includegraphics[width=0.45\linewidth]{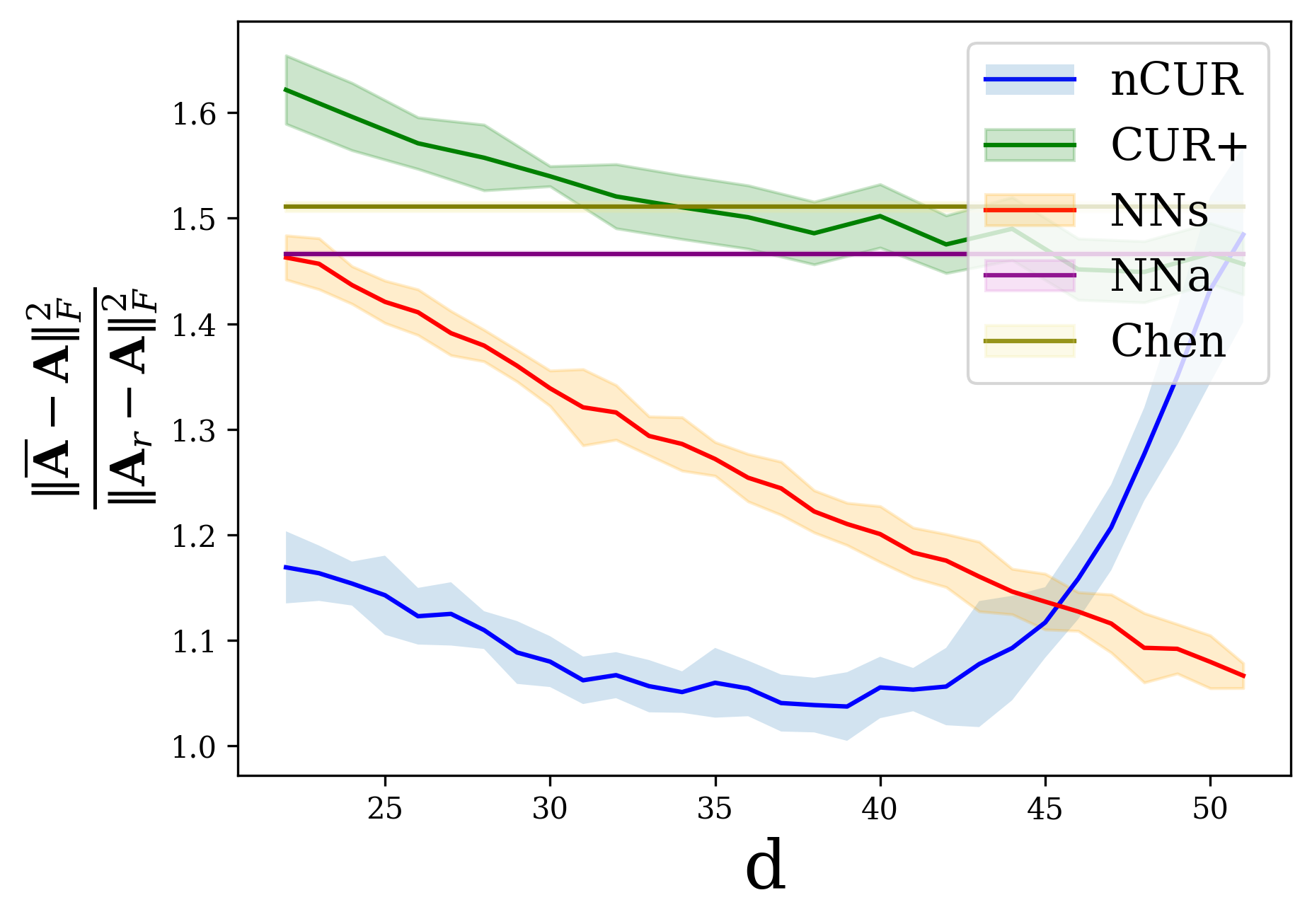} }%
    \subfloat{\includegraphics[width=0.45\linewidth]{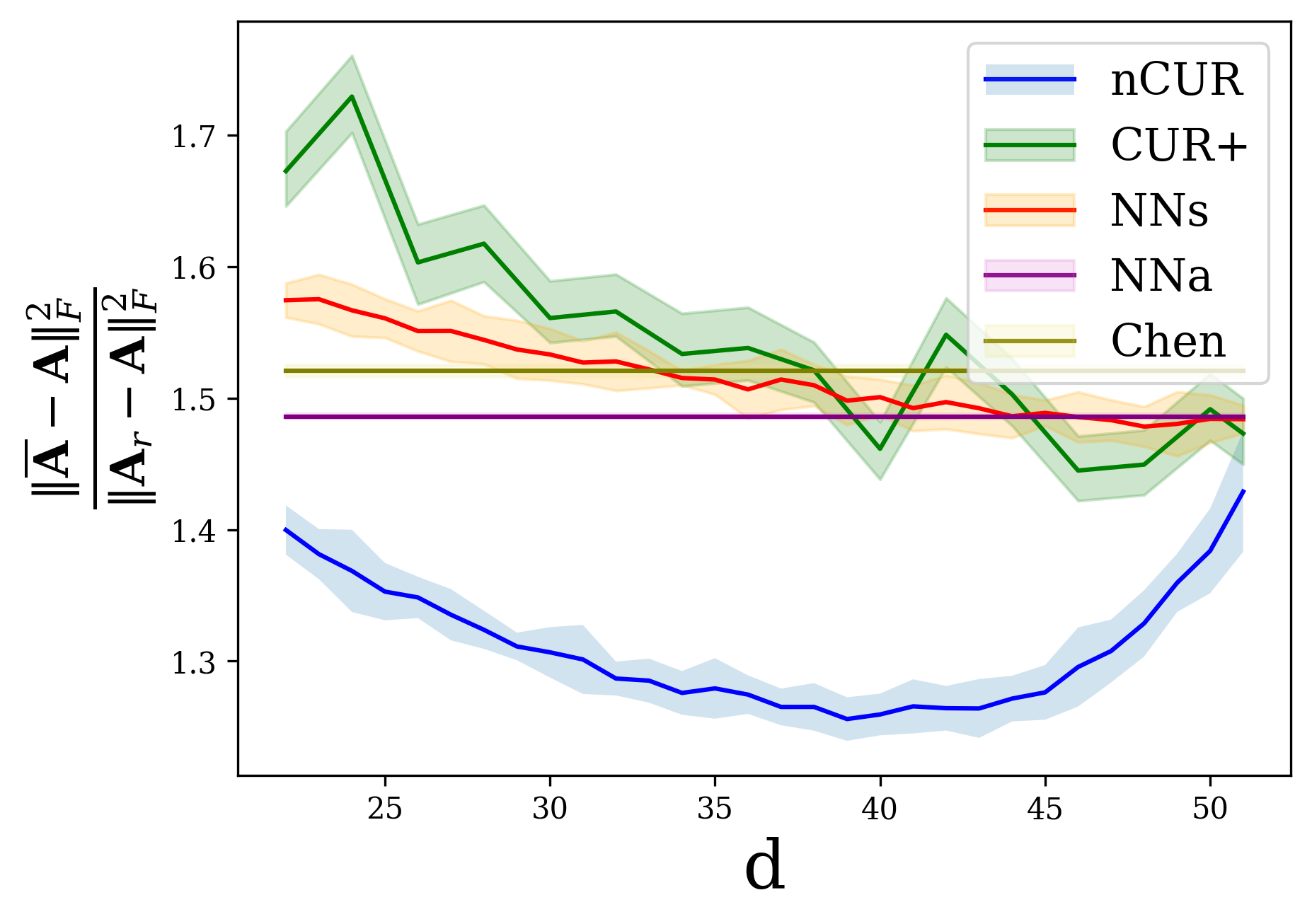} }%
    \caption{Performance on the Jester data set. $\beta=11\%$ for all plots. The noise level is low in the top plots: $\sigma_e^2=0.04,\sigma_c^2=2$. In the bottom two plots, the noise level is higher: $\sigma_e^2=0.25$, $\sigma_c^2=12.5$. As in Figure~\ref{fig:synthetic}, the plots to the left contain the NNs baseline while those to the right do not. Each point in the plots is the average over 10 runs. }%
    \label{fig:jester}%
\end{figure}	
    
Figure~\ref{fig:jester} compares the performance of the baseline methods and noisyCUR on a subset of the Jester dataset of~\cite{JESTER}. Specifically the data set was constructed by extracting the submatrix comprising the $7200$ users who rated all 100 jokes. For each value of $d$, the regularization parameter $\lambda$ of noisyCUR is selected via cross-validation from 200 points logarithmically spaced in the interval $(10, 10^5)$. 

\subsection{Movielens-100K}

\begin{figure}[htb!]%
    \centering
    \subfloat{\includegraphics[width=0.45\linewidth]{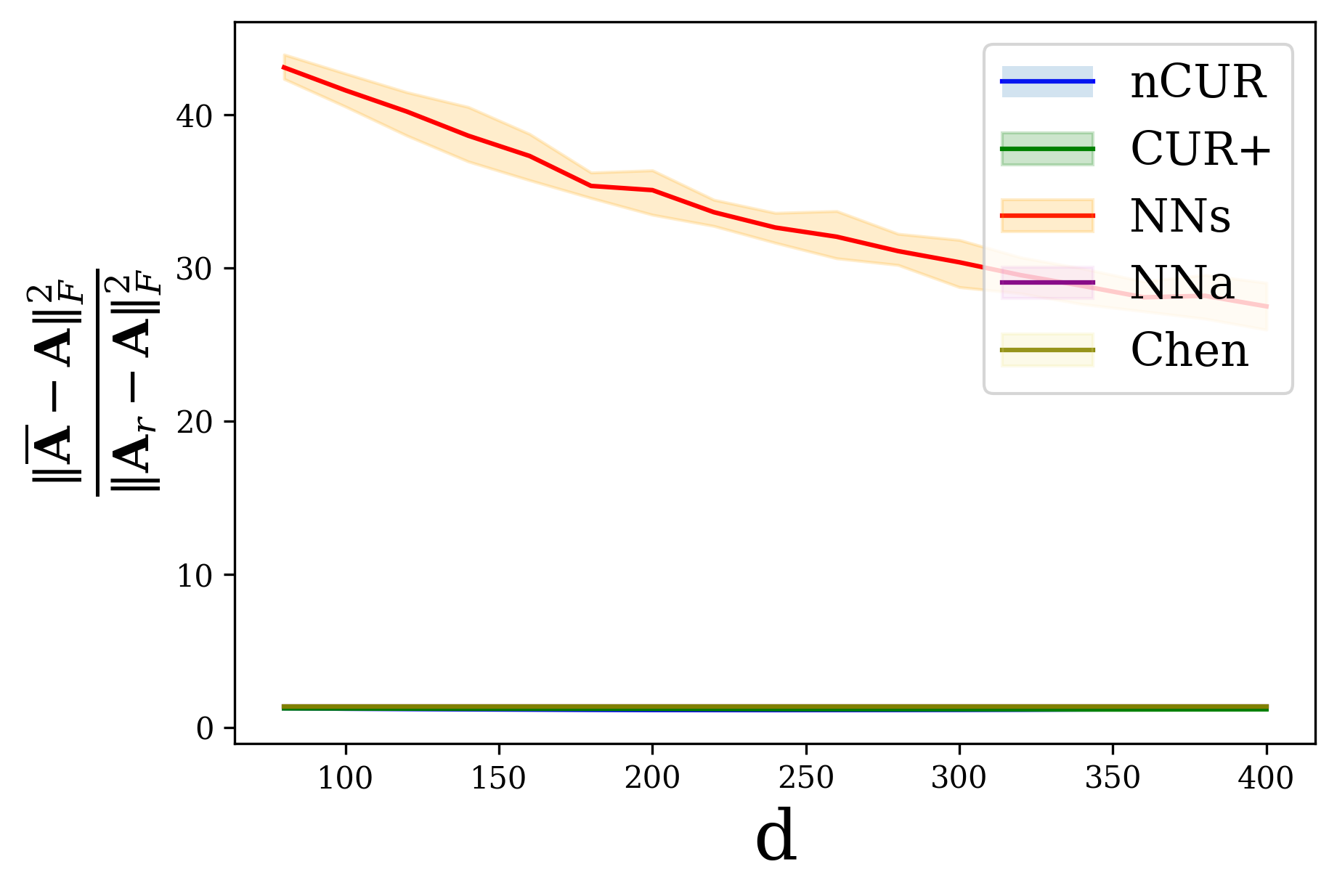} }%
    \subfloat{\includegraphics[width=0.45\linewidth]{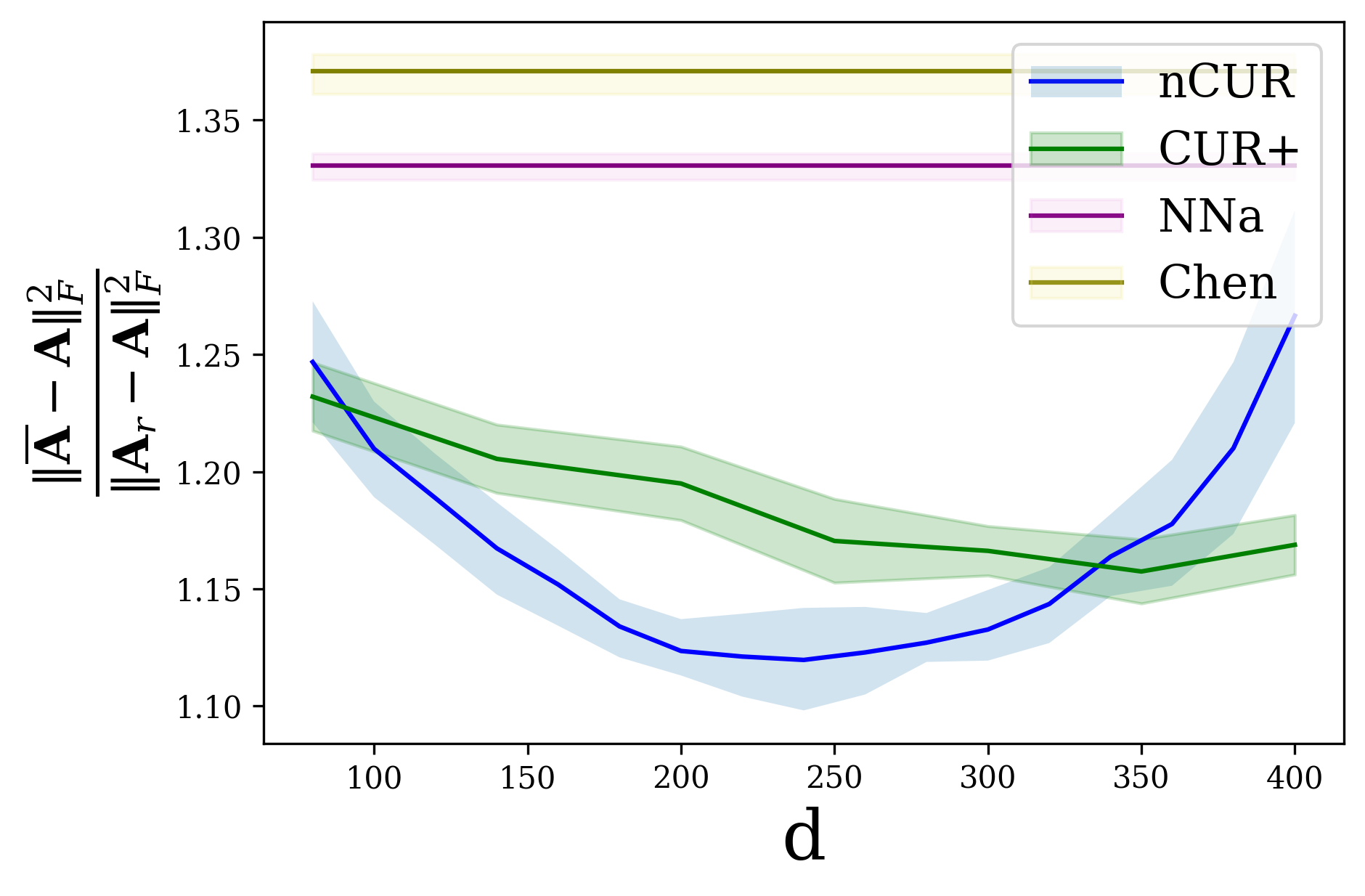} }%
   
    \subfloat{\includegraphics[width=0.45\linewidth]{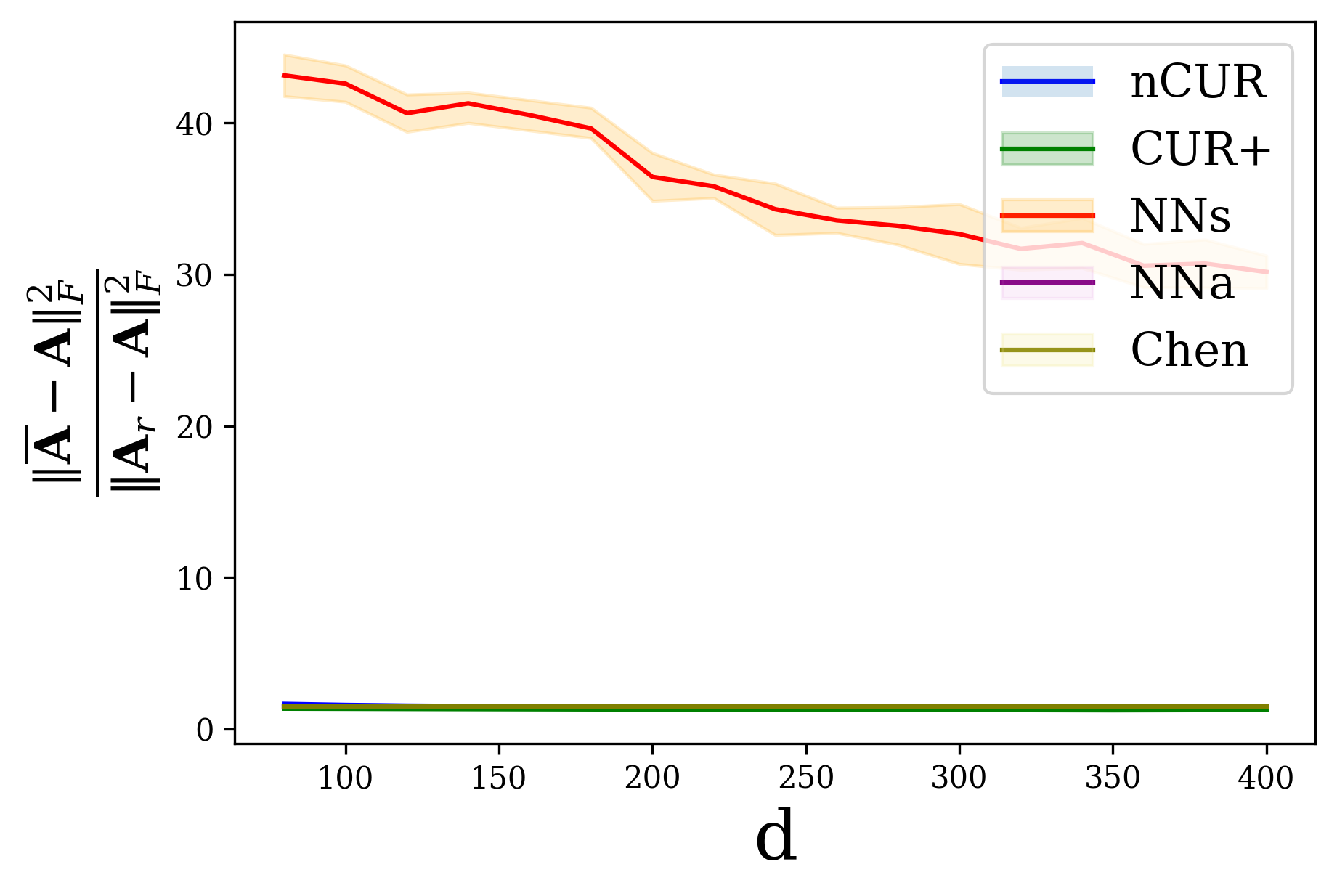} }%
    \subfloat{\includegraphics[width=0.45\linewidth]{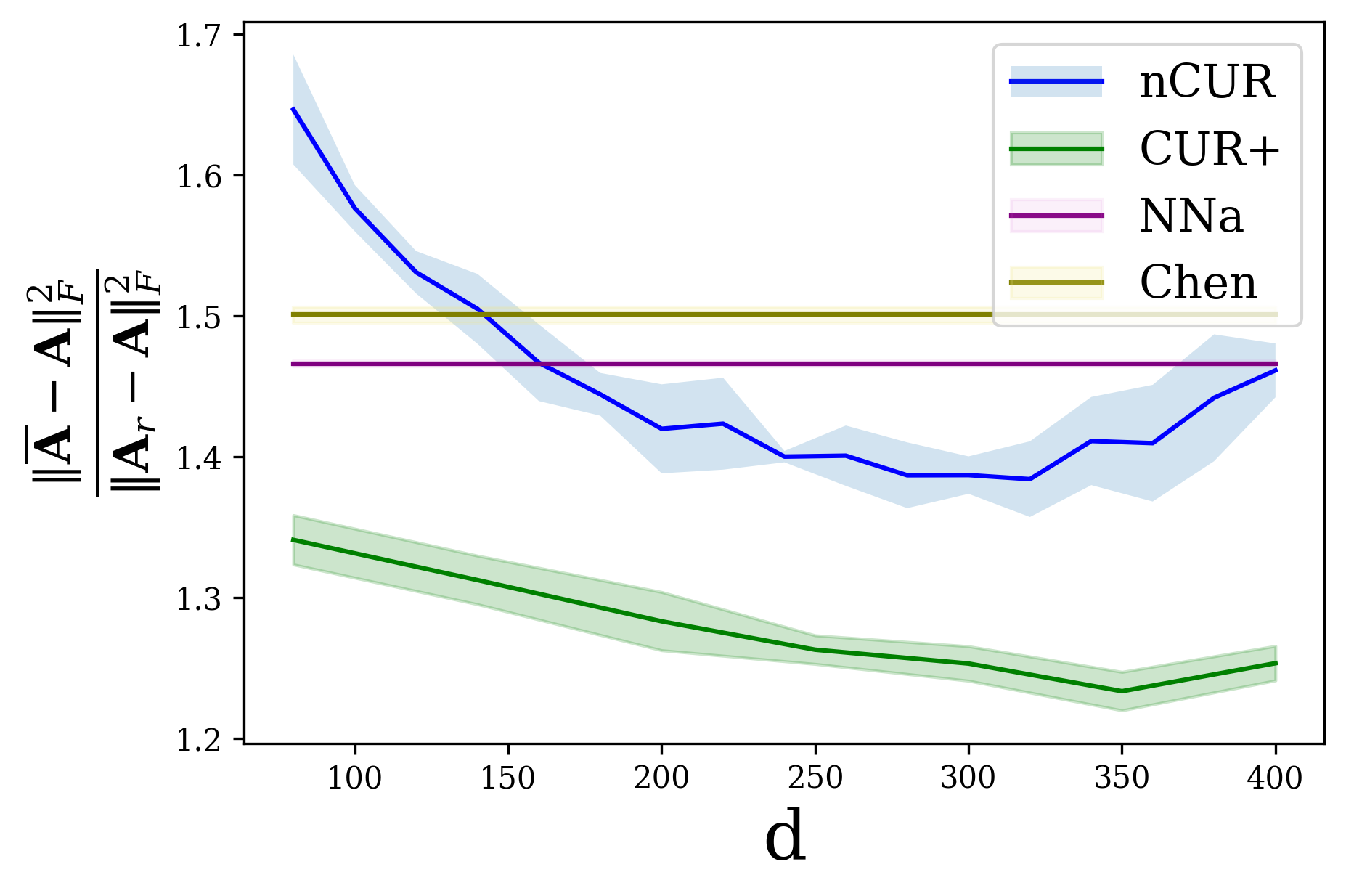} }%
    \caption{Performance on the MovieLens-100K data set. $\beta=10.6\%$ for all plots. The noise level is low in the upper two plots: $\sigma_e^2=0.003,\sigma_c^2=0.06$. In the bottom two plots, the noise level is higher: $\sigma_e^2=0.0225,\sigma_c^2=0.65$. As in Figure~\ref{fig:synthetic}, the plots to the left show the NNs baseline while the plots to the right do not. Each point in the plots is the average of 10 runs.}%
    \label{fig:movie}%
\end{figure}
 
 Figure~\ref{fig:movie} compares the performance of the baseline methods and noisyCUR on the Movielens-100K dataset of~\cite{MOVIE}. The original $1682\times943$ data matrix is quite sparse, so the iterative SVD algorithm of~\cite{SVD} is used to complete it to a low-rank matrix before applying noisyCUR and the three baseline algorithms. For each value of $d$, the parameter $\lambda$ of noisyCUR is cross-validated among 200 points linearly spaced in the interval $[1,200]$. 
 
\subsection{Observations} 
In both the lower and higher noise regimes on all the datasets, noisyCUR exhibits superior performance when compared to the nuclear norm baselines, and in all but one of the experiments, noisyCUR outperforms CUR+, the regression-based baseline. Also, noisyCUR produces a v-shaped error curve, where the optimal approximation error is achieved at the bottom of the v. This convex shape of the performance of noisyCUR with respect to $d$ suggests that there may be a single optimal $d$ given a dataset and the parameters of the two-cost model.

We note a practical consideration that arose in the empirical evaluations: the noisyCUR method is much faster than the nuclear norm algorithms, as they invoke an iterative solver (ADMM was used in these experiments) that computes SVDs of large matrices during each iteration, while noisyCUR solves a single ridge regression problem. 

\section{Conclusion}
\label{sec:conclusion}

This paper introduced the noisyCUR algorithm for solving a budgeted matrix completion problem where the two observation modalities consist of high-accuracy expensive entry sampling and low-accuracy inexpensive column sampling. Recovery guarantees were proven for noisyCUR; these hold even in low-budget regimes where standard nuclear norm completion approaches have no recovery guarantees. noisyCUR is fast, as the main computation involved is a ridge regression. Empirically, it was shown that noisyCUR has lower reconstruction error than standard nuclear norm completion baselines in the low-budget setting. It is an interesting open problem to determine optimal or near-optimal values for the number of column and row samples ($d$ and $s$), given the parameters of the two-cost model ($B$, $p_e$, $p_c$, $\sigma_e$, and $\sigma_c$). Implementations of noisyCUR and the baseline methods used in the experimental evaluations are available at \url{https://github.com/jurohd/nCUR}, along with code for replicating the experiments.

\bibliographystyle{splncs04}

\bibliography{twocost}
\end{document}

\title{Supplementary Material for ``NoisyCUR: An algorithm for two-cost budgeted matrix completion''}
\author{Dong Hu\inst{1} \and
Alex Gittens\inst{1} \and
Malik Magdon-Ismail\inst{1}}
% %
% \author{Anonymous}
% \institute{Anonymous}
\authorrunning{D. Hu et al.}

\institute{Rensselaer Polytechnic Institute, Troy, NY 12180, USA \\
\email{\{hud3,gittea\}@rpi.edu, magdon@cs.rpi.edu}}
\maketitle              % typeset the header of the contribution

\label{supp-sec:algorithm}

For the convenience of the reader we restate the noisyCUR algorithm and provide a full proof of Theorem~\ref{supp-thm:noisycur_guarantee}. We restate some results from the main body of the paper to minimize back-references.

\begin{algorithm}[h!]
 	\begin{algorithmic}[1]
 		\Require $d$, the number of column samples;
 		$s$, the number of row samples; $\sigma_e$, the noise level of the column samples; $\sigma_c$, the noise level of the row samples; and $\lambda$, the regularization parameter
 		\Ensure $\overline{\matA}$, approximation to $\matA$
		\State $\tilde{\matC} \gets \matC + \matE_c$, where the $d$ columns of $\matC$ are sampled uniformly at random with replacement from $\matA$, and the entries of $\matE_c$ are i.i.d. $\mathcal{N}(0, \sigma_c^2)$.
 		\State $\matU \gets $ an orthonormal basis for the column span of $\tilde{\matC}$
		\State $\bm{\ell}_i \gets \tfrac{1}{2} \|\e_i^T \matU\|_2^2/\|\matU\|_F^2 + \tfrac{1}{2m}$ for $i = 1, \ldots, m$
		\State $\matS \gets \text{SamplingMatrix}(\bm{\ell}, m, s)$, the sketching matrix used to sample $s$ rows of $\matA$
		\State $\matY = \matS^T \matA + \matE_e$, where the entries of $\matE_e$ are i.i.d. $\mathcal{N}(0, \sigma_e^2)$.
		\State $\matX \gets \argmin_{\matZ} \|\matY - \matS^T \tilde{\matC} \matZ\|_F^2 + \lambda \|\matZ\|_F^2$
		\State $\overline{\matA} \gets \tilde{\matC}\matX$
		\State \Return $\overline{\matA}$
	\end{algorithmic}
	\caption{Noisy CUR algorithm for completion of a low-rank matrix $\matA \in \R^{m \times n}$}
	 \label{supp-alg:noisycur}
 \end{algorithm}
 
In Algorithm~\ref{supp-alg:noisycur}, the $\text{SamplingMatrix}(\bm{\ell}, m, s)$ function returns a sketching matrix $\matS \in \R^{m \times s}$ such that $\matS^T \matA$ samples $s$ rows, i.i.d.~with replacement, from $\matA$ with probability proportional to their shrinked leverage scores. Let $p_i = \frac{1}{2} \left(\frac{\ell_i}{\sum_{j=1}^m \ell_j} + \frac{1}{m} \right)$ denote the shrinked leverage probability associated with row $i$. Then $\matS$ has exactly one non-zero per column, the location of that non-zero follows the distribution given by $p_1, \ldots, p_m$, and if $i_j$ is the location of the nonzero in the $j$th column then $S_{i_j, j} = \frac{1}{\sqrt{sp_{i_j}}}$. Thus left multiplication by $\matS^T$ samples and rescales rows from $\matA$ according to their shrinked leverage scores. \cite{ma2015statistical,wang2018sketched}~motivate the use of shrinked leverage score sampling.

The remainder of this supplement develops the following guarantee on the quality of the low-rank approximation returned by Algorithm~\ref{supp-alg:noisycur}.

\begin{thm}
\label{supp-thm:noisycur_guarantee}
Let $\matA \in \R^{m \times n}$ be a rank-$r$ matrix with $\beta$-incoherent column space. Assume that $\matA$ is dense: there is a $c > 0$ such that at least half\footnote{This fraction can be changed, with corresponding modification to the sample complexities $d$ and $s$.} of the entries of $\matA$ satisfy $|a_{ij}| \geq c$.

Fix a precision parameter $\varepsilon \in (0, 1)$ and invoke Algorithm~\ref{supp-alg:noisycur} with 
\[
d \geq \max\left\{ \frac{6 + 2 \varepsilon}{3\varepsilon^2} \beta r \log \frac{r}{\delta},  \frac{8(1 + \delta)^2}{c^2(1 - \varepsilon)\varepsilon} r \kappa_2(\matA)^2 \sigma_c^2\right\} 
\text{ and }
s \geq \frac{6 + 2 \varepsilon}{3\varepsilon^2} 2d \log \frac{d}{\delta}.
\]
The returned approximation satisfies
\[
\|\matA - \overline{\matA}\|_F^2 \leq \left(\gamma + \varepsilon + 40 \frac{\varepsilon}{1 - \varepsilon}\right) \|\matA\|_F^2 + 12 \varepsilon \left(\frac{\sigma_e^2}{\sigma_c^2}\right) d \sigma_r^2(\matA)
\]
with probability at least $0.9 - 2\delta - 2\exp(\tfrac{-(m-r)\delta^2}{2}) - \exp(\tfrac{-sn}{32})$. Here,
\[
\gamma \leq 2 \left(\frac{1 + \varepsilon}{1 - \varepsilon}\right) \left[ \frac{\lambda}{(1 + \varepsilon) \left( \frac{1}{2}\sqrt{m-r} - \sqrt{d} \right)^2\sigma_c^2 + \lambda}\right]^2.
\]
\end{thm}

Theorem~\ref{supp-thm:noisycur_guarantee} is a consequence of two structural results, the proofs of which are deferred. 

The first result states that if $\text{rank}(\matC) = \text{rank}(\matA)$ and the bottom singular value of $\matC$ is large compared to $\sigma_c$, then the span of $\tilde{\matC}$ will contain a good approximation to $\matA$. 

\begin{lem}
	\label{supp-lem:structural_noisy_approx}
	Fix an orthonormal basis $\matU \in \R^{m \times r}$ and consider $\matA \in \R^{m \times n}$ and $\matC \in \R^{m \times d}$ with factorizations $\matA = \matU \matM$ and $\matC = \matU \matW$, where both $\matM$ and $\matW$ have full row rank. Further, let $\tilde{\matC}$ be a noisy observation of $\matC$, that is, let $\tilde{\matC} = \matC + \matG$ where the entries of $\matG$ are i.i.d. $\mathcal{N}(0, \sigma_{c}^2)$. If $\sigmamin(\matC) \geq 2(1 + \delta) \sigma_c \sqrt{m/\varepsilon}$, then 
	\[
		\|(\matI - \matP_{\tilde{\matC}})\matA\|_F^2 \leq \varepsilon \|\matA\|_F^2
		\]
	with probability at least $1 - \exp\left(\frac{-m\delta^2}{2}\right)$.
\end{lem}

Recall the definition of a $(1\pm \varepsilon)$-subspace embedding.

\begin{defn}[Subspace embedding~\cite{woodruff2014sketching}]
	Let $\matA \in \R^{m \times n}$ and fix $\varepsilon \in (0,1)$. A matrix $\matS \in \R^{m \times s}$ is a $(1 \pm \varepsilon)$-subspace embedding for $\matA$ if
	\[
		(1 - \varepsilon) \|\x\|_2^2 \leq \|\matS^T\x\|_2^2 \leq (1 + \varepsilon) \|\x\|_2^2
	\]
	for all vectors $\x$ in the span of $\matA,$ or equivalently, if
	\[
	(1 - \varepsilon) \matA^T \matA \preceq \matA^T \matS \matS^T \matA \preceq (1 + \varepsilon) \matA^T \matA.
	\]
	Often we will use the shorthand ``subspace embedding'' for $(1 \pm \varepsilon)$-subspace embedding.
\end{defn}
We note that the $\matS$ returned by $\text{SamplingMatrix}(\bm{\ell}, m, s)$ in Algorithm~\ref{supp-alg:noisycur} is a subspace embedding with high probability~\cite[Appendix A.1.1]{wang2018sketched}.

The second structural result is a novel bound on the error of sketching using a subspace embedding to reduce the cost of ridge regression, when the target is noisy.

\begin{cor}
  \label{supp-cor:structural_sketched_proximal_regularized_LS}
  Let $\tilde{\matC} \in \R^{m \times d}$, where $d \leq m$, and $\tilde{\matA} = \matA + \matE$ be matrices, and let $\matS$ be an $(1 \pm \varepsilon)$-subspace embedding for $\tilde{\matC}$. If
	\[
    \matX = \argmin_{\matZ} \|\matS^T(\tilde{\matA} - \tilde{\matC} \matZ)\|_F^2 + \lambda \|\matZ\|_F^2,
	\]
	then 
	\begin{multline*}
    \|\matA - \tilde{\matC} \matX\|_F^2 \leq \|(\matI - \matP_{\tilde{\matC}})\matA\|_F^2 + \gamma \|\matP_{\tilde{\matC}} \matA\|_F^2
       + \frac{4}{1 - \varepsilon}\|\matS^T \matE\|_F^2 + \frac{4}{1 - \varepsilon}\|\matS^T (\matI - \matP_{\tilde{\matC}})\matA\|_F^2,
	\end{multline*}
  where $\gamma =  2 \left(\frac{1 + \varepsilon}{1 - \varepsilon} \right) \left( \frac{\lambda}{(1 + \varepsilon) \sigma_d(\tilde{\matC})^2 + \lambda} \right)^{2}.$
\end{cor}

Corollary~\ref{supp-cor:structural_sketched_proximal_regularized_LS} differs significantly from prior results on the error in sketched ridge regression, e.g.~\cite{avron2017sharper,wang2018sketched}, in that: (1) it bounds the \emph{reconstruction error} rather than the \emph{ridge regression objective}, and (2) it considers the impact of noise in the target. This result follows from a more general result on sketched noisy proximally regularized least squares problems, stated as Theorem~\ref{supp-thm:sketched_regularized_LS} below.

Together with standard properties of Gaussian noise and subspace embeddings, these two results deliver Theorem~\ref{supp-thm:noisycur_guarantee}.

\begin{proof}[Proof of Theorem~\ref{supp-thm:noisycur_guarantee}]
The NoisyCUR algorithm first forms the noisy column samples $\tilde{\matC} = \matC + \matE_c$, where $\matC = \matA \matM$.  The random matrix $\matM \in \R^{n \times d}$ selects $d$ columns uniformly at random with replacement from the columns of $\matA$, and the entries of $\matE_c \in \R^{m \times d}$ are i.i.d.\ $\mathcal{N}(0, \sigma_c^2)$. It then solves the sketched regression problem
\[
\matX = \argmin_{\matZ} \|\matS^T(\tilde{\matA} - \tilde{\matC} \matZ)\|_F^2 + \lambda \|\matZ\|_F^2,
\]
and returns the approximation $\overline{\matA} = \tilde{\matC} \matX$. Here $\tilde{\matA} = \matA + \matE_e$, where $\matE_e \in \R^{m \times n}$ comprises i.i.d $\mathcal{N}(0, \sigma_e^2)$ entries, and the sketching matrix $\matS \in \R^{m \times s}$ samples $s$ rows using the shrinked leverage scores of $\tilde{\matC}$.

By~\cite[Appendix A.1.1]{wang2018sketched}, $\matS$ is a subspace embedding for $\tilde{\matC}$ with failure probability at most $\delta$ when $s$ is as specified. Thus Corollary~\ref{supp-cor:structural_sketched_proximal_regularized_LS} applies and gives that
\begin{align*}
    \|\matA - \tilde{\matC} \matX\|_F^2 & \leq \|(\matI - \matP_{\tilde{\matC}})\matA\|_F^2 + \gamma^\prime \|\matP_{\tilde{\matC}} \matA\|_F^2 \\
      & \quad \quad + \frac{4}{1 - \varepsilon}\|\matS^T \matE\|_F^2 + \frac{4}{1 - \varepsilon}\|\matS^T (\matI - \matP_{\tilde{\matC}})\matA\|_F^2 \\
       & = T_1 + T_2 + T_3 + T_4,
\end{align*}
where $\gamma^\prime =  2 \left(\frac{1 + \varepsilon}{1 - \varepsilon} \right) \left( \frac{\lambda}{(1 + \varepsilon) \sigma_d(\tilde{\matC})^2 + \lambda} \right)^{2}.$ We now bound the four terms $T_1$, $T_2$, $T_3$, and $T_4$. 

To bound $T_1$, note that by~\cite[Lemma 13]{wang2016towards}, the matrix $\sqrt{\tfrac{n}{d}} \matM$ is a subspace embedding for $\matA^T$ with failure probability at most $\delta$ when $d$ is as specified. This gives the semidefinite inequality $\frac{n}{d} \matC \matC^T = \frac{n}{d} \matA \matM \matM^T \matA^T \succeq (1 - \varepsilon) \matA\matA^T$, which in turn gives that
\begin{align*}
\sigma_{r}^2(\matC) & \geq (1 - \varepsilon) \frac{d}{n} \sigma_r^2(\matA) 
\geq \frac{8(1+\delta)^2}{c^2 \varepsilon} \frac{r}{n} \|\matA\|_2^2 \sigma_c^2 \\
& \geq \frac{8(1+\delta)^2}{c^2 \varepsilon n} \|\matA\|_F^2 \sigma_c^2 \geq 4(1+ \delta)^2\frac{m}{\varepsilon} \sigma_c^2. 
\end{align*}
The second inequality holds because 
\begin{equation}
\label{supp-eqn:lowerbound_bottom_singval}
d \geq \frac{8(1 + \delta)^2}{c^2(1 - \varepsilon)\varepsilon} r \kappa_2(\matA)^2 \sigma_c^2 \quad\text{implies}\quad 
\sigma_r^2(\matA) \geq \frac{8(1 + \delta)^2}{c^2(1 - \varepsilon)\varepsilon} \frac{r}{d} \|\matA_2\|^2 \sigma_c^2
\end{equation}
The third inequality holds because $r \|\matA\|_2^2$ is an overestimate of $\|\matA_F\|_2^2$. The final inequality holds because the denseness of $\matA$ implies that $\|\matA\|_F^2 \geq \tfrac{1}{2} c^2 mn$.

Note also that the span of $\matC = \matA \matM$ is contained in that of $\matA$, and since $\frac{n}{d} \matC \matC^T \succeq (1 - \varepsilon) \matA \matA^T$, in fact $\matC$ and $\matA$ have the same rank and therefore span the same space. Thus the necessary conditions to apply Lemma~\ref{supp-lem:structural_noisy_approx} are satisfied, and as a result, we find that
\[
T_1 \leq \varepsilon \|\matA\|_F^2
\]
with failure probability at most $\exp(-\tfrac{m\delta^2}{2}).$

Next we bound $T_2$. Observe that $\|\matP_{\tilde{\matC}} \matA\|_F^2 \leq \|\matA\|_F^2$. Further, by Lemma~\ref{supp-lem:perturbed_bottom_sing_val},
\[
\sigma_d(\tilde{\matC}) \geq \left(\frac{1}{2}\sqrt{m-r} - \sqrt{d}\right)^2 \sigma_c^2
\]
with failure probability at most $\exp(\tfrac{-(m-r)\delta^2}{2})$. This allows us to conclude that
\[
T_2 \leq \gamma \|\matA\|_F^2,
\]
where $\gamma$ is as specified in the statement of this theorem.

To bound $T_3$, we write
\begin{align*}
T_3 & = \frac{4}{1-\varepsilon}\|\matS^T \matP_{\matS} \matE\|_F^2 \leq \frac{4}{1-\varepsilon}\|\matS\|_2^2 \|\matP_{\matS} \matE\|_F^2 \\
& \leq \frac{8}{1-\varepsilon} \frac{m}{s} \|\matQ^T \matE\|_F^2,
\end{align*}
where $\matQ$ is an orthonormal basis for the span of $\matS$.
The last inequality holds because~\cite[Appendix A.1.2]{wang2018sketched} shows that $\|\matS\|_2^2 \leq 2\tfrac{m}{s}$ always. Finally, note that $\matQ$ has at most $s$ columns, so in the worst case $\matQ^T\matE$ comprises $sn$ i.i.d.\ $\mathcal{N}(0, \sigma_e^2)$ entries. A standard concentration bound for $\chi^2$ random variables with $sn$ degrees of freedom~\cite[Example 2.11]{wainwright2019high} guarantees that
\[
\|\matQ^T \matE\|_F^2 \leq \frac{3}{2} sn \sigma_e^2 
\]
with failure probability at most $\exp(\tfrac{-sn}{32})$. We conclude that, with the same failure probability,
\[
T_3 \leq \frac{12}{1- \varepsilon} m n \sigma_e^2. 
\]
Now recall \eqref{supp-eqn:lowerbound_bottom_singval}, which implies that
\begin{align*}
\varepsilon(1-\varepsilon) d\sigma_r^2(\matA) 
& \geq \frac{8(1 + \delta)^2}{c^2} r \|\matA_2\|^2 \sigma_c^2 
\geq \frac{8(1 + \delta)^2}{c^2} \|\matA\|_F^2 \sigma_c^2 \\
& \geq 4(1 + \delta)^2 m n \sigma_c^2 \geq m n \sigma_c^2.
\end{align*}
It follows from the last two displays that
\[
T_3 \leq 12 \varepsilon \left(\frac{\sigma_e^2}{\sigma_c^2} \right) d \sigma_r^2(\matA).
\]

The bound for $T_4$ is an application of Markov's inequality. In particular, it is readily verifiable that $\mathbb{E}[\matS\matS^T] = \matI$, which implies that
\[
\mathbb{E}T_4 = \frac{4}{1- \varepsilon}\|(\matI - \matP_{\tilde{\matC}})\matA \|_F^2 = \frac{4}{1- \varepsilon}T_1 \leq \frac{4\varepsilon}{1- \varepsilon}\|\matA\|_F^2.
\]
The final inequality comes from the bound $T_1 \leq \varepsilon \|\matA\|_F^2$ that was shown earlier. Thus, by Markov's inequality, 
\[
T_4 \leq \frac{40\varepsilon}{1- \varepsilon}\|\matA\|_F^2
\]
with failure probability at most $0.1$.

Collating the bounds for $T_1$ through $T_4$ and their corresponding failure probabilities gives the claimed result.
\end{proof}

\begin{proof}[Proof of Lemma~\ref{supp-lem:structural_noisy_approx}]
	Consider the matrix $\matZ = \tilde{\matC} \matW^\dagger \matM$ in the span of $\tilde{\matC}$. Because $\matW$ has full row-rank, $\matZ = \matU \matW \matW^\dagger \matM + \matG \matW^\dagger\matM = \matA + \matG\matW^\dagger\matM$. Thus 
	\begin{align*}
		\Prob\left( \|(\matI - \matP_{\tilde{\matC}})\matA\|_F^2 \geq \varepsilon \|\matA\|_F^2 \right) & \leq \Prob\left(\|\matZ - \matA\|_F^2 \geq \varepsilon \|\matA\|_F^2\right) \\
		& = \Prob\left( \|\matG \matW^\dagger \matM\|_F^2 \geq \varepsilon \|\matM\|_F^2 \right) \\
		& \leq \Prob\left( \|\matG\|_2^2 \|\matW^\dagger\|_2^2 \|\matM\|_F^2 \geq \varepsilon \|\matM\|_F^2 \right) \\
		& = \Prob\left(\|\matG\|_2^2 \geq \varepsilon \|\matW^\dagger\|_2^{-2}\right) \\
		& = \Prob\left(\|\matG\|_2^2 \geq \varepsilon \sigmamin(\matW)^2 \right) \\
		& \leq \Prob\left(\|\matG\|_2 \geq 2 (1 + \delta) \sigma_c \sqrt{m}\right) \\
		& \leq \Prob\left(\|\matG\|_2 \geq \sigma_c \left[ (1 + \delta) \sqrt{m} + \sqrt{d}\right] \right)\\
		& \leq \exp\left(\frac{-m\delta^2}{2}\right).
	\end{align*}
The last inequality is a consequence of the concentration of the singular values of Gaussian matrices~\cite[Theorem 6.1]{wainwright2019high}. 
\end{proof}

The next result quantifies the impact of sketching on noisy proximally regularized least squares problems. 

\begin{thm}
	\label{supp-thm:sketched_regularized_LS}
  Fix a matrix $\matA \in \R^{m \times d}$, where $d \leq m$, vectors $\b$, $\e$, and $\x$, and a regularization parameter $\lambda > 0$. 
  Let $\matS$ be a subspace embedding for the matrix $\matA$ and consider the objective 
  functions $f : \z \mapsto \|\matA \z - \b\|_2^2$ and $f_{\matS, \e} : \z \mapsto \|\matS^T(\matA \z - \b - \e)\|_2^2$. If 
	\[
    \x_{\matS,\e}^\lambda = \argmin_{\z} f_{\matS, \e} (\z) + \lambda \|\z - \x\|_2^2,
		\]
then 
	\[
    f(\x_{\matS, \e}^\lambda) \leq f(\matA^\dagger \b) + \gamma (f(\x) - f(\matA^\dagger \b)) + \frac{4}{1 - \varepsilon} \|\matS^T \e\|_2^2 + \frac{4}{1 - \varepsilon} \|\matS^T (\matI - \matP_{\matA})\b\|_2^2,
		\]
where $\gamma = 2 \left(\frac{1 + \varepsilon}{1 - \varepsilon} \right) \left( \frac{\lambda}{(1 + \varepsilon) \sigma_d(\matA)^2 + \lambda} \right)^{2}.$
\end{thm}

\begin{rem}
  The gap $f(\x)$ - $f(\matA^\dagger \b)$ between the error at the proximal point and the optimal error is a measure of the goodness of the proximal point $\x$.
  Thus this theorem shows that the error $f(\x_{\matS, \e}^\lambda)$ of the resulting estimate comes from readily identifiable sources:
  the best possible error $f(\matA^\dagger \b)$, a measure of the quality of the goodness of the proximal point, and measures of how sketching interacts with the noise
  in the target and the portion of the target than cannot be captured in the column space of $\matA$.
\end{rem}

We defer the proof of Theorem~\ref{supp-thm:sketched_regularized_LS} until after the proof of Corollary~\ref{supp-cor:structural_sketched_proximal_regularized_LS}.

\begin{proof}[Proof of Corollary~\ref{supp-cor:structural_sketched_proximal_regularized_LS}]
	This result follows by applying Theorem~\ref{supp-thm:sketched_regularized_LS} to the separable least squares problems defining the columns of $\matX$.
  Namely, let $\x_i$, $\a_i$, and $\e_i$ denote the $i$th columns of $\matX$, $\matA$, and $\matE$ respectively, then 
  \[
    \x_i = \argmin_\z \|\matS^T(\a_i + \e_i - \tilde{\matC} \z)\|_2^2 + \lambda\|\z - 0 \|_2^2,
  \]
  and Theorem~\ref{supp-thm:sketched_regularized_LS} guarantees that
  \begin{align*}
    \|\a_i - \tilde{\matC} \x_i \|_2^2  & \leq \|(\matI - \matP_{\tilde{\matC}})\a_i\|_2^2 
     + \gamma( \|\a_i\|_2^2 - \|(\matI - \matP_{\tilde{\matC}})\a_i\|_2^2)  \\
     &  \quad \quad + \frac{4}{1-\varepsilon}\|\matS^T \e_i\|_2^2 
     + \frac{4}{1 -\varepsilon}\|\matS^T (\matI - \matP_{\tilde{\matC}})\a_i\|_2^2 \\
     & =  \|(\matI - \matP_{\tilde{\matC}})\a_i\|_2^2 + \gamma( \|\matP_{\tilde{\matC}}\a_i\|_2^2)  \\
     & \quad \quad + \frac{4}{1-\varepsilon}\|\matS^T \e_i\|_2^2 
     + \frac{4}{1 -\varepsilon}\|\matS^T (\matI - \matP_{\tilde{\matC}})\a_i\|_2^2.
  \end{align*}
  Summing over the errors in estimating all the columns of $\matA$ gives the claimed result.
\end{proof}

We use the following technical lemma in the proof of Theorem~\ref{supp-thm:sketched_regularized_LS}. The proof is deferred.
\begin{lem}
  \label{supp-lem:awkward_sketch}
  Fix a matrix $\matA \in \R^{m \times d}$, where $d \leq m$,  and let $\matS$ be a subspace embedding for $\matA$. If $\lambda > 0$ and $\v \in \R^d$, then
  \[
    \|\matA(\matA^T \matS \matS^T \matA + \lambda \matI)^{-1} \v\|_2^2 \leq \frac{1 + \varepsilon}{1 - \varepsilon} \left( \frac{1}{(1 + \varepsilon)\sigma_d(\matA)^2 + \lambda} \right)^2 \|\matA \v\|_2^2.
  \]
\end{lem}

\begin{proof}[Proof of Theorem~\ref{supp-thm:sketched_regularized_LS}]
Solving the noisy sketched proximally regularized problem gives the coefficients
  \begin{align}
    \label{supp-eqn:sketched_coeff_expression}
    \x_{\matS, \e}^\lambda & = \argmin_{\z} [f_{\matS, \e}(\z) + \lambda \|\z - \x\|_2^2 ] \notag\\
     & = (\matA^T \matS \matS^T \matA + \lambda \matI)^{-1}(\matA^T \matS\matS^T (\b  + \e) + \lambda \x) \notag\\
     & = (\matA^T \matS \matS^T \matA + \lambda \matI)^{-1}(\matA^T \matS\matS^T (\matP_{\matA} \b + (\matI - \matP_{\matA}) \b  + \e) + \lambda \x) \notag \\
     & = (\matA^T \matS \matS^T \matA + \lambda \matI)^{-1}(\matA^T \matS\matS^T \matA \matA^\dagger \b + \lambda \x) + \n_{\matS},
  \end{align}
where $\n_{\matS} = (\matA^T \matS \matS^T \matA + \lambda \matI)^{-1} \matA^T \matS \matS^T (\e + (\matI - \matP_{\matA}) \b).$
 
The error of using these coefficients to recover the true target relates to the error of the least-squares coefficients $\matA^\dagger \b$ as follows: 
\begin{align*}
  f(\x_{\matS, \e}^\lambda) & = \|\matA \x_{\matS, \e}^\lambda - \b\|_2^2 = \|\matA \x_{\matS, \e}^\lambda - \matA \matA^{\dagger}\b\|_2^2 + \|(\matI - \matP_{\matA})\b\|_2^2 \\
   & =  \|\matA \x_{\matS, \e}^\lambda - \matA \matA^{\dagger}\b\|_2^2 + f(\matA^\dagger \b).
\end{align*}
Insert the expression \eqref{supp-eqn:sketched_coeff_expression} for $\x_{\matS,\e}^\lambda$ into this expression and exploit the invertibility of $\matA^T \matS \matS^T \matA + \lambda \matI$ to observe that
\begin{align*}
  f(\x_{\matS, \e}^\lambda) & = \|\matA (\matA^T \matS \matS^T \matA + \lambda \matI)^{-1} (\matA \matS \matS^T \matA \matA^\dagger \b + \lambda \x) + \matA \n_{\matS} \\
   & \quad \quad - \matA (\matA^T \matS \matS^T \matA + \lambda \matI)^{-1} (\matA^T \matS \matS^T \matA + \lambda \matI) \matA^\dagger \b\|_2^2 
   + f(\matA^\dagger \b) \\
   & = \|\lambda \matA(\matA^T \matS \matS^T \matA + \lambda \matI)^{-1} (\x - \matA^\dagger \b) + \matA \n_{\matS} \|_2^2 + f(\matA^\dagger \b).
\end{align*}
A standard algebraic manipulation on the first term, followed by an application of Lemma~\ref{supp-lem:awkward_sketch} gives
\begin{align*}
  \label{supp-eqn:sketched_loss_intermed}
  f(\x_{\matS, \e}^\lambda) & \leq 2\lambda^2 \|\matA(\matA^T \matS \matS^T \matA + \lambda \matI)^{-1} (\x - \matA^\dagger \b) \|_2^2 + 2\|\matA \n_{\matS} \|_2^2 + f(\matA^\dagger \b) \notag \\
  & \leq \gamma \|\matA \x - \matP_{\matA} \b\|_2^2 + f(\matA^\dagger \b) + 2\|\matA \n_{\matS}\|_2^2,
\end{align*}
where for notational convenience we have introduced 
\[
  \gamma = 2 \left(\frac{1 + \varepsilon}{1 - \varepsilon} \right) \left( \frac{\lambda}{(1 + \varepsilon) \sigma_d(\matA)^2 + \lambda} \right)^{2}.
\]

Now recall that $f(\x) = \|\matA - \matP_{\matA}\b\|_2^2 + f(\matA^\dagger \b)$; this implies that
\begin{equation}
  \label{supp-eqn:almostthere_loss}
  f(\x_{\matS, \e}^\lambda) \leq \gamma f(\x) +  (1 - \gamma) f(\matA^\dagger \b) + 2 \|\matA \n_{\matS}\|_2^2.
\end{equation}

All that remains is to estimate the noise term $\|\matA \n_{\matS}\|_2^2$ in \eqref{supp-eqn:almostthere_loss}. To do so, we first 
exploit the fact that $\matS$ is a subspace embedding for $\matA$:
\begin{align*}
  \|\matA \n_{\matS}\|_2^2 & \leq \frac{1}{1 - \varepsilon} \|\matS^T \matA \n_{\matS}\|_2^2 \\
  & \leq \frac{1}{1 - \varepsilon} \|\matS^T \matA (\matA^T \matS \matS^T \matA + \lambda \matI)^{-1} \matA^T \matS\|_2^2 \|\matS^T(\e + (\matI - \matP_{\matA})\b)\|_2^2 \\
  & = \nu \|\matS^T(\e + (\matI - \matP_{\matA})\b)\|_2^2,
\end{align*}
where $\nu =  (1 - \varepsilon)^{-1} \|\matS^T \matA (\matA^T \matS \matS^T \matA + \lambda \matI)^{-1} \matA^T \matS\|_2^2$.
To estimate $\nu$, observe that 
since $\matA^T \matS \matS^T \matA + \lambda \matI \succeq \matA^T \matS \matS^T \matA$, the inequality 
\[
  \matS^T \matA \mat(\matA^T \matS \matS^T \matA + \lambda \matI)^{-1} \matA^T \matS \preceq \matS^T \matA (\matA^T \matS \matS^T \matA)^\dagger \matA^T \matS = \matP_{\matS^T \matA} 
\]
holds. It follows that
\[
  \nu \leq \frac{1}{1 - \varepsilon} \|\matP_{\matS^T \matA}\|_2^2 = \frac{1}{1 - \varepsilon},
\]
consequently
\begin{align*}
  \|\matA \n_{\matS} \|_2^2 & \leq \frac{1}{1 - \varepsilon} \|\matS^T (\e + (\matI - \matP_{\matA}) \b)\|_2^2 \\ 
  & \leq \frac{2}{1 - \varepsilon} \left( \|\matS^T \e\|_2^2 + \|\matS^T (\matI - \mat{P}_{\matA}) \b\|_2^2 \right).
\end{align*}

\end{proof}

\begin{proof}[Proof of Lemma~\ref{supp-lem:awkward_sketch}]
Since $\matS$ is a subspace embedding for $\matA$, we have that
  \begin{align}
    \label{supp-eqn:monotonicity1}
    \|\matA(\matA^T \matS \matS^T \matA + \lambda \matI)^{-1} \v\|_2^2 & \leq \frac{1}{1 - \varepsilon} \|\matS^T \matA(\matA^T \matS \matS^T \matA + \lambda \matI)^{-1} \v\|_2^2 \notag \\
     & = \frac{1}{1 - \varepsilon} \v^T g(\matA^T \matS \matS^T \matA) \v,
  \end{align}
where the function $g : (0, \infty) \rightarrow (0, \infty)$ is given by $g : t \mapsto \frac{t}{(t + \lambda)^2}$.
  
  Notice that $1/g : t \mapsto t + 2\lambda + \lambda^2/t$. The first two terms are trivially operator convex, and the last term is also operator convex, as $t^{-1}$ is operator convex~\cite[Proposition 2]{chansangiam2015survey}.
\cite[Corollary 12(iv)]{chansangiam2015survey} states that a function $h: (0, \infty) \rightarrow (0, \infty)$ is operator monotone iff $1/h$ is operator convex, so we conclude that $g$ is operator monotone.

  Because $\matS$ is a subspace embedding for $\matA$, it is the case that $\matA^T \matS \matS^T \matA \preceq (1 + \varepsilon) \matA^T \matA$.  The monotonicity of $g$ then implies that $g(\matA^T \matS \matS^T \matA) \preceq g( (1 + \varepsilon) \matA^T \matA)$, and as a consequence,
\begin{align}
  \label{supp-eqn:monotonicity2}
  \v^T g(\matA^T \matS \matS^T \matA) \v 
   & \leq \v^T g( (1 + \varepsilon )\matA^T \matA) \v \notag\\ 
  &  = \frac{1}{1+ \varepsilon} \v^T \left(\matA^T \matA + \frac{\lambda}{1 + \varepsilon} \matI \right)^{-1} \matA^T \matA  \left(\matA^T \matA + \frac{\lambda}{1 + \varepsilon} \matI \right)^{-1} \v \notag \\
  &  = \frac{1}{1+\varepsilon} \|\matA(\matA^T \matA + \lambda^\prime \matI)^{-1} \v\|_2^2,
\end{align}
where for brevity's sake we have introduced the notation $\lambda^\prime = \tfrac{\lambda}{1 + \varepsilon}$.

  It follows from \eqref{supp-eqn:monotonicity1} and \eqref{supp-eqn:monotonicity2} that
  \begin{align*}
    \|\matA(\matA^T \matS \matS^T \matA + \lambda \matI)^{-1} \v\|_2^2 & \leq \frac{1}{1 - \varepsilon^2} \|\matA(\matA^T \matA + \lambda^\prime \matI)^{-1} \v\|_2^2 \\
     & = \frac{1}{1 - \varepsilon^2} \|(\matA \matA^T + \lambda^\prime \matI)^{-1} \matA \v\|_2^2 \\
     & \leq \frac{1}{1 - \varepsilon^2} \left( \frac{1}{\sigma_d(\matA)^2 + \lambda^\prime} \right)^2 \|\matA \v\|_2^2.
  \end{align*}
  The first equality is a standard algebraic identity, and the last inequality is justified by identifying the norm of $(\matA\matA^T + \lambda^\prime \matI)^{-1}$. Inserting
   the definition of $\lambda^\prime$ and simplifying delivers the claimed result,
   \[
     \|\matA(\matA^T \matS \matS^T \matA + \lambda \matI)^{-1} \v\|_2^2 \leq \frac{1 + \varepsilon}{1 - \varepsilon}\left(\frac{1}{(1+\varepsilon)\sigma_d(\matA)^2 + \lambda}\right)^2\|\matA\v\|_2^2 \\
   \]
\end{proof}

We use the following lemma in the proof of Theorem~\ref{supp-thm:noisycur_guarantee}.

\begin{lem}
\label{supp-lem:perturbed_bottom_sing_val}
  Let $\matC \in \R^{m \times d}$, where $d \leq m$, have rank $r$, and let the conformal matrix $\matE$ comprise i.i.d. $\mathcal{N}(0, \sigma^2)$ entries,
  then 
  \[
    \sigma_d(\matC + \matE)^2 \geq \left(\frac{1}{2}\sqrt{m-r} - \sqrt{d}\right)^2 \sigma^2
  \]
  with probability at least $1 - \exp(\tfrac{-(m-r)\delta^2}{2})$.
\end{lem}

\begin{proof}
  Observe that
  \begin{align}
    \label{supp-eqn:bottom_sing_val_estimate}
    \sigma_d(\matC + \matE)^2 & = \min_{\|\x\|_2 = 1} \|(\matC + \matE) \x\|_2^2 = \min_{\|\x\|_2=1} \|(\matC + \matP_{\matC}\matE)\x\|_2^2 + \|(\matI - \matP_{\matC}) \matE\x\|_2^2\notag \\
    & \geq \min_{\|\x\|_2=1} \|(\matI - \matP_{\matC}) \matE \x\|_2^2 \notag \\
    & = \min_{\|x\|_2=1} \|\matQ^T \matE \x\|_2^2 = \sigma_{\text{min}}(\matQ^T \matE)^2,
  \end{align}
   where $\matQ \in \R^{m \times (m-r)}$ is an orthonormal basis for the kernel of $\matP_{\matC}$. Since $\matQ^T \matE$ comprises i.i.d. $\mathcal{N}(0, \sigma^2)$ entries, standard results
   on the singular values of Gaussian matrices give that
   \[
     \sigma_{\text{min}}(\matQ^T\matE)^2 \geq \left(\frac{1}{2}\sqrt{m-r} - \sqrt{d}\right)^2 \sigma^2.
   \]
   with probability at least $1 - \exp(-\tfrac{(m-r)\delta^2}{2})$~\cite[Theorem 6.1]{wainwright2019high}. The claimed result now follows from inequality \ref{supp-eqn:bottom_sing_val_estimate}.
\end{proof}

\bibliographystyle{splncs04}
\bibliography{twocost}